\RequirePackage{fix-cm}
\documentclass[twocolumn]{svjour3}
\smartqed  
\usepackage{graphicx}
\usepackage{mathptmx}      
\usepackage{latexsym}
\usepackage{import}
\usepackage{times}
\usepackage{epsfig}
\usepackage{amsmath}
\usepackage{amssymb}
\usepackage{color}
\usepackage{bbm}
\usepackage{ulem}
\usepackage{algorithm}
\usepackage{algorithmic}
\newcommand{\tableScaleSize}{0.23}

\renewcommand{\tabcolsep}{0.01cm}
\newcommand{\specialcell}[2][c]{%
  \begin{tabular}[#1]{@{}c@{}}#2\end{tabular}}

\newcommand{\norm}[1]{\left\Vert#1\right\Vert}

\newcommand{\set}[1]{\left\{#1\right\}}

\newcommand{\ip}[1]{\left \langle #1 \right \rangle }
\newcommand{\Real}{\mathbb R}

\newcommand{\too}{\rightarrow}

\newcommand{\N}{\mathcal{N}}

\newcommand{\wt}[1]{\widetilde{#1}} 
\newcommand{\wh}[1]{\widehat{#1}} 

\newcommand{\HS}{\wh{\mathcal{S}}}

\newcommand{\goal}[1]{}

\newcommand{\newtext}[1]{{\color{black} #1}}
\newcommand{\no}[1]{{\color{black} #1}}

\journalname{Journal of Mathematical Imaging and Vision}

\begin{document}

\title{Photometric Stereo by Hemispherical Metric Embedding
}

\author{Ofer Bartal \and Nati Ofir \and Yaron Lipman \and Ronen Basri}


\institute{Ofer Bartal \at
           Weizmann Institute of Science \\
           \email{ofer.bartal@weizmann.ac.il}
           \and
           Nati Ofir \at
           Weizmann Institute of Science \\
          \email{yehonatan.ofir@weizmann.ac.il}
          \and
           Yaron Lipman \at
           Weizmann Institute of Science \\
           \email{yaron.lipman@weizmann.ac.il}
           \and
           Ronen Basri \at
           Weizmann Institute of Science \\
           \email{ronen.basri@weizmann.ac.il}           
}

\date{Received: 10.8.16 / Accepted: 21.6.17}

\maketitle

\begin{abstract}
Photometric Stereo methods seek to reconstruct the 3d shape of an object from motionless images obtained with varying illumination. Most existing methods solve a restricted problem where the physical reflectance model, such as Lambertian reflectance, is known in advance. 
In contrast, we do not restrict ourselves to a specific reflectance model. Instead, we offer a method that works on a wide variety of reflectances. Our approach uses a simple yet uncommonly used property of the problem - the sought after normals are points on a unit hemisphere. We present a novel embedding method that maps pixels to normals on the unit hemisphere. Our experiments demonstrate that this approach outperforms existing manifold learning methods for the task of hemisphere embedding. We further show successful reconstructions of objects from a wide variety of reflectances including smooth, rough, diffuse and specular surfaces, even in the presence of significant attached shadows. Finally, we empirically prove that under these challenging settings we obtain more accurate shape reconstructions than existing methods.
\end{abstract}

\section{Introduction}
\label{intro}

Photometric stereo (PS) methods aim to recover the shape of objects from collections of motionless images taken with varying illumination. PS is challenging due to a number of reasons. In particular, object reflectance can range from Lambertian to specular with varying degrees of roughness. In addition, surrounding lights may be complex and include multiple light sources along with reflections from surrounding objects. Furthermore, in many practical cases neither the light setting nor the reflectance properties will be known. While most existing methods (see review in Section~\ref{sec:previous_work}) apply PS in restricted settings, for example, assuming Lambertian reflectance or point source lighting, there is growing interest in developing approaches to PS that may be able to handle this problem under more general settings.

This paper takes a generic view of PS. PS can be thought of as the problem of discovering a mapping (or embedding) of intensity measurements to surface normals. Each pixel is associated with a vector of its intensity values in the $n$ input images. Our approach aims to embed these intensity vectors onto the hemisphere of forward facing normal vectors using only local distances between the intensity vectors.

We then use the surface normals to recover a depth map of the object by integration, as done in \cite{RonenSH2}.
Our main contribution in this work is in introducing a novel metric embedding technique which is specifically designed to target the hemisphere. We show that our technique outperforms existing embedding techniques for the task of hemisphere embedding. We further show empirically that by using some general statistical assumptions on the lighting, our embedding technique allows the recovery of shapes under a variety of lighting conditions and reflectance properties, and is robust to attached shadows. We illustrate this on real images by showing reconstructions of Lambertian as well as non-Lambertian surfaces, acquired under unknown complex lighting conditions that include point sources as well as varying ambient illumination.

Generic embedding methods of high-dimensional point clouds are well known in the machine learning and multi-dimensional scaling (MDS) literature as effective methods for dimensionality reduction (e.g.,~\cite{IsomapPaper,LLE,MDS}). These methods map a manifold residing in high-dimensional space into a low dimensional Euclidean target space, which is typically chosen to be the intrinsic dimension of the manifold. Embedding to a hemisphere is inherently different: while the intrinsic dimension of a hemisphere is 2, its ambient or ``target'' dimension is 3, as we wish to preserve the hemisphere structure and not flatten it into a 2 dimensional space. We show that for existing methods, embedding a hemisphere manifold onto 3 dimensions leads to distortions. For some algorithms, these distortions occur even in the trivial identity mapping of a hemisphere that already lies in a 3 dimensional space. Our method is specifically designed for this kind of embedding and allows direct recovery of surface normals. 

To realize this embedding we use distances between intensity vectors to approximate the Laplace-Beltrami (LB) differential operator over the hemisphere, \newtext{ by solving an optimization problem}. We then make the observation that with suitable boundary conditions on the hemisphere's equator, the eigenfunctions of the LB operator are the spherical harmonic functions. The linear spherical harmonics correspond \textit{exactly} to the X,Y, and Z coordinates of the normals on the hemisphere. \newtext{The idea of spherical harmonics representation of lighting is described in details in \cite{RonenSH1}}. The key observation being that since the Laplacian is a \textit{local} operator, good approximation of \textit{local} distances leads to a robust \textit{global} embedding over the hemisphere.

\section{Previous work}
\label{sec:previous_work}

\no{Early methods of photometric stereo focused on handling Lambertian surfaces under known distant point light sources \cite{Woodham1980}. More advanced methods address the case of a known \cite{Xie_2015_CVPR} or unknown \cite{Uncalibrated-Near-Light} nearby point light source. In addition, other works solve the problem of unknown distant point light sources and Lambertian surface. \cite{Hayakawa1994} solves the problem up to linear ambiguity. By further enforcing integrability, ambiguity reduces to GBR \cite{BasReliefAmbiguity,YuilleIntegrability}. \cite{EntropyMinimizationPS,LDR,Queau2015,SelfCalibratingPS} use a prior on surface albedo or shape to solve the GBR, while \cite{PapadhimitriF13} solves it using a perspective camera model.
Further work generalizes lighting to multi point or arbitrary lighting \cite{RonenSH2,Queau2015_multipoint}, or generalizes to non-Lambertian reflectance with either known \cite{SeitzExampleBasedPS,ConsensusPS} or unknown lighting \cite{Georghiades2003,Athos,Sato2,Midorikawa_2016_CVPR,AttachedShadowCoding,SatoICCV2007,YiMaPS}.
\cite{HeliometricStereo,WebcamPS} extends PS to outdoor webcam images and \cite{Gotardo_2015_ICCV} couples PS with other techniques.
For integrating normals, \cite{Durou2009,Frankot88amethod,Simchony} use a depth map, while \cite{Ho2016,Queau_2016_CVPR,mecca2015realistic,Tozza2016} avoid integration by following a direct differential approach to PS. This idea was also applied to other computer vision problems \cite{Smith2016}. \cite{Durou2009,durou2016} presents robust methods for surface integration in the presence of discontinuities.}

Sato et al.~\cite{SatoICCV2007} have used a generic embedding technique for photometric stereo, employing the Isomap algorithm~for the embedding \cite{IsomapPaper}. 
In \cite{Sato2}, Lu et al. employ an iterative approach that is based on matrix decomposition with missing data. Their method requires the estimation of a reflectance-based scaling parameter. 
They provide extensive experiments to show the behavior of this parameter under various reflectances. We complement their work by providing a proof for the existence of this parameter for 
general \newtext{isotropic, band limited and single lobe reflectance kernels that explains and further validates their empirical findings. These band limited kernels are typically diffuse kernels with no specular component.}
In addition, we take a general hemisphere embedding technique approach that avoids the need to estimate any such reflectance-based parameter. Furthermore, we show that empirically our method achieves more accurate reconstructions on a variety of simulated and real experiments. 

Hemisphere and sphere embeddings have also been found useful in different contexts. \newtext{Dimensionality reduction into a (full) sphere was used in \cite{sLLE} for the problem of tomographic reconstruction from unknown viewing angles, and in \cite{littwin2015spherical} for the problem of spherical embeddings of silhouettes}. They restrict the Locally Linear Embedding \cite{LLE} into the sphere by formulating the problem as a hard to optimize quadratic program with non-convex cubic constraints. In contrast, our method embeds into the (hemi) sphere via solving an eigenvector problem and is guaranteed to find the correct embedding up to approximation errors without the risk of finding a local minimum. \newtext{Moreover, it does not suffer from challenging conditions such as attached-shadows and varying albedo.}

\section{Approach}
\label{sec:Our approach}
\goal{We want a mapping from intensities to normals}
Our approach views Photometric Stereo as a problem of metric embedding. We are given a set of input images \(I_{1}, I_{2}\dots I_{n}\), with $N$ pixels, of the same object seen from the same viewing direction but illuminated with different \newtext{directional light sources}. For each pixel $p$ we want to find its normal $n_p \in \HS$, where $\HS=\set{(x,y,z) \ | \ x^2+y^2+z^2=1, z\geq 0}$ denotes the unit hemisphere representing the set of forward facing surface normals (below we call this {\em the normal hemisphere}). These normals $\set{n_p}$ form a discrete sampling of the normal hemisphere $\HS$. For each pixel $p$ we denote by the vector $v_p$ the set of intensities observed at that pixel over the $n$ images.

\newtext{It has been shown in~\cite{SatoICCV2007,Sato2} that when the images are produced with single directional light sources whose direction is distributed uniformly over the unit sphere and whose intensity is constant}, normalized $\ell_2$ distances of \textit{nearby} intensity vectors approximate fairly well the corresponding spherical distances between the normals, namely
$$\|\hat v_p - \hat v_q\|_2 \approx d(n_p,n_q),$$ where \(\hat v_p=v_p/\norm{v_p}_2\) and \(\hat v_q=v_q/\norm{v_q}_2\). \no{Since the albedo of the surface is a multiplicative factor of the intensities, this normalization factors it out. Note that we use the notion of albedo to refer to the diffuse component of more complex reflectances. The albedo is an attribute of the object, not of the lighting, and can be changed from pixel to pixel. Note that since we approximate the spherical distance between nearby normals, the shapes that can be recovered by our approach are those with enough samples of normals from all the sides of the hemisphere.}
We extend the analysis done in \cite{SatoICCV2007} to reflectances produced by an \newtext{isotropic, band limited and single lobe} reflectance kernel. \newtext{This extends this result to images obtained with general light source configurations and to objects made of some non-Lambertian materials. With this setup we} prove the following claim: 
\begin{claim}{}\label{sh_isotropic_reflectance_proof_paper} 
Under uniformly distributed \newtext{directional} light sources and an isotropic, band limited \newtext{and single lobe} reflectance kernel, the $\ell_2$ distances of normalized intensity vectors $\hat v_p$ and $\hat v_q$ of points $p$ and $q$ with \textit{nearby} normals $n_p$ and $n_q$ are proportional to the spherical distance $d(n_p,n_q)$ between the corresponding surface normals, i.e.,

\begin{equation}  \label{eq:sph-dist}
\|\hat v_p - \hat v_q\|_2 = c \cdot d(n_p,n_q),
\end{equation}

where the constant $c$ depends on the reflectance kernel.
\end{claim}
\begin{proof}{}
See appendix.
\end{proof}

\newtext{
Our proof relies on a spherical harmonic decomposition of the reflectance function. \cite{RonenSH1,Ramamoorthi&Hanrahan2001} showed that images of Lambertian objects can be well approximated by a convolution of a low frequency kernel with the incoming lighting function. Formally, let $\ell(u): S^2 \rightarrow \Real$ represent incoming light intensity as a function of direction (a vector on the unit sphere $S^2$) then the image intensity $I(p)$ at a pixel $p$ with normal $n \in S^2$ and albedo $\rho$ can be expressed as
\begin{equation}
I(p) = \rho \int_{u \in S^2} k(u,n) \ell(u) du,
\end{equation}
where $k(u,n) = \max(u^T n, 0)$ is the {\em Lambertian kernel}. They further showed that $k$ acts as a low frequency kernel, and so all images of a Lambertian object, under arbitrary complex lightings, are well approximated by a spherical harmonic expansion with just 9 low-order terms. 

Our claim applies to Lambertian objects illuminated by arbitrary light sources when the set of light source directions over the entire set of images is distributed uniformly over the unit sphere. It further applies to non-Lambertian objects whose reflectance can be described with an isotropic, single-lobe band limited kernel , see e.g.~\cite{Nillius2004} for a list of such materials. Note that in either case it also accounts for attached (self) shadows.}

Our claim compliments the extensive empirical experiments conducted in \cite{Sato2} that show that a constant as in~\eqref{eq:sph-dist} empirically exists for a wide variety of reflectances. Unlike \cite{Sato2}, we do not need to estimate this parameter. Its affect on the hemisphere manifold structure is a uniform global scaling, effectively making it a smaller or a larger hemisphere. As we take a general hemisphere embedding approach, we are indifferent to this kind of uniform scaling. Nonetheless, we prove the following claim:
\begin{claim}{}\label{sh_isotropic_reflectance_proof_paper_Lambertian} 
The reflectance-dependent constant $c$ can be computed analytically for Lambertian reflectance, and is approximately~0.93.
\end{claim}
\begin{proof}{}
See appendix for proof. Intuitively, one might expect $c$ to equal 1 for Lambertian reflectance. Since Lambertian reflectance is non-smoothly clamped at zero \newtext{(recall that the Lambertian reflectance kernel, $k(u,n) = \max(u^T n, 0)$, is clamped at zero due to attached shadows}, intensity differences become smaller than they would have been for an unclamped Lambertian reflectance function. This effect brings the intensity distance of two pixels closer compared to their geodesic distance. The magnitude of this effect is analytically computed in our proof.

\end{proof}
Similarly to \cite{SatoICCV2007}, we define the neighborhood $\N_p=\set{q_1,...,q_k}$ of a pixel $p$ by taking the $k$-nearest neighbors w.r.t. the metric $\|\hat v_p - \hat v_q\|_2$. Locally, this neighborhood should be similar to the neighborhood of $n_p$ defined on the hemisphere using the spherical distances.

\renewcommand{\tableScaleSize}{.49}
\begin{figure}[t]
\center
\begin{tabular}{cc}
\includegraphics[width=\tableScaleSize\columnwidth]{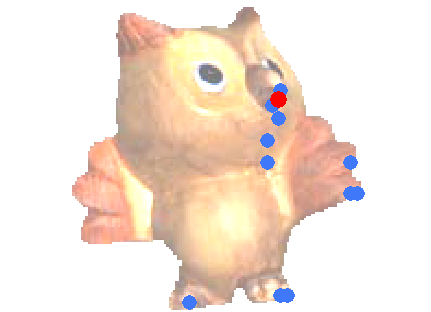}
&\includegraphics[width=\tableScaleSize\columnwidth]{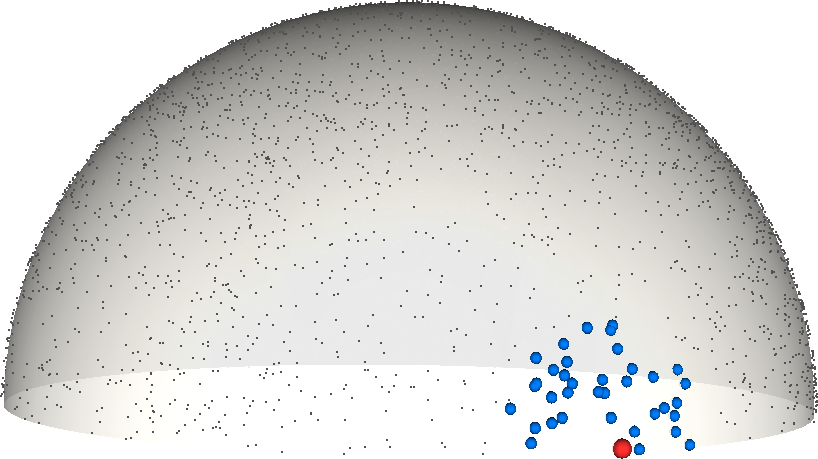}
\end{tabular}
\caption{\textbf{A neighborhood in the image and on the hemisphere: }Left: A pixel in the owl object marked in red and its neighborhood in blue. Right: The same pixel and neighborhood shown on the normal hemisphere. Far away pixels in the image can map to nearby points on the hemisphere and vice versa.}
\label{fig:NeighborhoodProcessFig}
\end{figure}

Figure~\ref{fig:NeighborhoodProcessFig} shows an example neighborhood of pixels. As can be seen, pixels with close normals may be spread far away in terms of image coordinates. Also, it may happen that many pixels in the image occupy a small area on the normal hemisphere.

This local pixel information, like pieces of a puzzle, needs to be combined in some form to obtain the global hemisphere structure. We present a novel hemisphere embedding approach to tackle this problem. The outline for our approach is as follows. First, we use local distances between points on the manifold to build an approximation of the Laplace-Beltrami (LB) differential operator $\Delta$ over the hemisphere $\HS$ while incorporating suitable boundary conditions. We then compute specific eigenfunctions of this operator. Finally, we show how to use these eigenfunctions to recover the unknown normals $n_p$. As motivation, we begin by considering our embedding based on the LB operator in the continuous case and later on, in Section~\ref{sec:algo}, describe our approach in the discrete case.

\goal{Continuous domain, have only local distances}

\section{Continuous Laplace-Beltrami operator}
\goal{Outline of approach}
The LB operator is a differential operator that is \textit{local}. This means that although we do \textit{not} know yet where each $n_p$ is located on the hemisphere, knowing its neighbors and the spherical distances to them is enough to define the LB operator $\Delta$. Over the full unit sphere $\mathcal{S}=\set{(x,y,z) \ | \ x^2+y^2+z^2=1}$, the first four eigenfunctions (e.g., $\phi$ that satisfy $\Delta \phi=\lambda \phi$ with smallest value of $\lambda$ is the first eigenfunction), also known as spherical harmonics are (up-to normalizing constants): $\phi_0(n_p)=1$, $\phi_1(n_p)=n_p^x$, $\phi_2(n_p)=n_p^y$, and $\phi_3(n_p)=n_p^z$, \color{black}see \cite{hobson1931theory}\color{black}. If the pixels' normals $n_p$ lay on the full sphere, we could approximate the LB operator over the sampling $\set{n_p}$ of the sphere and take the second to fourth eigenfunctions, set them as coordinates, and get the solution up-to a global rotation and reflection. Since we are dealing with a hemisphere $\HS$, we have to set the conditions on the equator of the hemisphere $\partial\HS=\set{(x,y,z)\ \mid \ x^2+y^2=1, z=0}$ appropriately in order to retain the spherical harmonics as eigenfunctions.

\goal{Describe the LB operator, needing only local distances}

\goal{Usage of eigenfunctions}

\goal{Setting boundary conditions}
Let $\Delta_D$ be the LB operator, $\Delta$, with Dirichlet boundary condition. Setting these conditions by forcing $\phi\mid_{\partial \HS}=0$ and solving for eigenfunctions $\Delta_D \phi=\lambda \phi$ will produce spherical harmonics that vanish on the equator. In particular, the first eigenfunction would be the Z-coordinate.
Similarly, let $\Delta_N$ be the LB operator with Neumann boundary condition. Setting these conditions to force the normal derivative at the equator to vanish
will produce all eigenfunctions that have this property over the equator. We denote the resulting operator $\Delta_N$. In particular, the first eigenfunction will be the constant function, followed by the X and Y coordinates, \no{up to rotation and reflection}.
More generally, it can be shown that the set of spherical harmonics $Y_{nm}$, $n \ge 0$ and $-n \le m \le n$ can be split into two sets. The harmonics with odd $m+n$, which form the eigenfunctions of the Dirichlet LB operator and the harmonics with even $m+n$, which form the eigenfunctions of the Neumann LB operator.

\goal{Eigenfunctions to normals}
We calculate the first eigenfunction of $\Delta_D$ and call it $\wt{n}_p^{z}$. Similarly, we calculate the second and third (first two non-constant) eigenfunctions of $\Delta_N$ and denote them by $\wt{n}_p^x$ and $\wt{n}_p^y$, respectively. We now claim that up to a rotation about the Z-axis (possibly with a reflection in the XY-plane) $\wt{n}_p=(\wt{n}_p^x,\wt{n}_p^y,\wt{n}_p^z)=Tn_p$. We later show how this orthogonal transformation $T$ can be resolved using surface integrability constraints, \color{black}and due to the normalization that factors out albedo\color{black}.


\goal{conclude}
Using this continuous setting as motivation we develop an algorithm that approximates the normals of pixels by first constructing a discrete version of the LB operator and then using its eigenvectors to recover the normals. We proceed by elaborating on each part of our algorithm, the outline of which was provided above.

\section{Algorithm} \label{sec:algo}

Drawing motivation from the continuous case, we next present an algorithm to approximate the pixels' normals, and in turn the depth-map, from intensity vectors, see Alg.~\ref{alg:main}. The local distances in the point cloud $\set{\hat{v}_p}\subset\Real^n$ provide an approximation to the spherical distances between the corresponding (unknown) cloud of normal vectors $\set{n_p}$. Large distances, however, may deviate significantly from their corresponding spherical distances. We therefore think of this point cloud $\set{\hat{v}_p}$ as an (approximate) isometric (i.e., local length preserving) embedding of the normal hemisphere vectors $\set{n_p}\subset \HS$ into a higher ($n$-dimensional) Euclidean space. Our goal is to embed the point cloud $\set{\hat{v}_p}$ back into its natural domain $\HS$ by using only the local distances.

\subsection{Outline of the Photometric Stereo Algorithm}

The outline of our algorithm is as follows:

\begin{enumerate}
  \item Approximate local spherical distances and neighborhoods over the (unknown) cloud of normals $\set{n_p}\subset \HS$ using local distances extracted from $\set{\hat{v}_p}\subset \Real^n$. See Alg.~\ref{alg:2}.
  \item Use these \textit{local} spherical distances to build a global least-squares optimization problem whose solution is the discrete Laplace-Beltrami matrix $L \approx \Delta$ over the (unknown) cloud of normals $\set{n_p}$. See Alg.~\ref{alg:3}.
  \item Apply appropriate Dirichlet boundary conditions and Neumann boundary conditions to the global optimization problem and solve it to obtain the matrices $L_D\approx \Delta_D, L_N\approx \Delta_N$, respectively. See Alg.~\ref{alg:4}.
  \item Recover the approximated normals $\wt{n}_p=(\wt{n}^x_p,\wt{n}^y_p,\wt{n}^z_p)$ for every pixel $p$ by calculating the first eigenvector $\wt{n}^z=\set{\wt{n}^z_p}$ of $L_D$, and the two first non-constant eigenvectors of $L_N$, $\wt{n}^x$ and $\wt{n}^y$. Use the approximated normals to recover the depth map $z(x,y)$. See Alg.~\ref{alg:5}.
\end{enumerate}
Next we describe each part of this algorithm in detail.

%


\begin{algorithm}
\caption{$PhotometricStereo(I_1,...,I_n)$}
\label{alg:main}
\begin{algorithmic}
  \REQUIRE Image sequence $I_1,...,I_n$, each image is of the same size of $N$ pixels, and under different lightings.
    \STATE $\{\mathcal{N}_p\}_{p=1}^N,\{\hat v_p\}_{p=1}^N \gets Neighbourhoods(I_1,...,I_n)$
    \STATE $L \gets LaplacianApproximation(\{\mathcal{N}_p\}_{p=1}^N,\{\hat v_p\}_{p=1}^N)$
    \STATE $L_D,L_N \gets BoundaryConditions(L)$
    \STATE $z(x,y)\gets RecoverDepth(L_D,L_N)$
    \RETURN $z(x,y)$
\end{algorithmic}
\end{algorithm}

\subsection{Local distances and neighborhoods}

Since the LB operator is local, i.e., the value of the (continuous) LB operator at a point, $\Delta \Phi (n_p)$, is determined by values $\Phi(n_q)$ at (infinitesimally) close points $n_q$, we will only need approximations of spherical distances between nearby pixel's normals $n_p \approx n_q$. These will be taken as the $k$ nearest neighbors of $\hat{v}_p$ and then made symmetric by an AND relation. That is, two pixels are maintained as neighbors if they are neighbors of each other.
\no{Due to the reconstruction of local neighborhoods, the shapes that can be recovered by our approach are those with enough samples of normals from all the sides of the hemisphere. In particular, smooth shapes whose boundaries in the image correspond to occluding contours of the surface satisfy this condition.}

\begin{algorithm}
\caption{$Neighbourhoods(I_1,...,I_n)$}
\label{alg:2}
\begin{algorithmic}
  \REQUIRE Image sequence $I_1,...,I_n$, each image is of the same size of $N$ pixels, and under different lightings.
    \FOR {$\forall p = (x,y)$}
    \STATE $v_p \gets (I(x,y,1),...,I(x,y,n))$
    \STATE $\hat v_p \gets v_p/||v_p||_2$
    \ENDFOR 
    \STATE $\forall p,q: d_{p,q} \gets ||\hat v_p-\hat v_q||_2$
    \STATE $\forall p: \mathcal{N}_p \gets \{p,q_1,...,q_k\}$ 
    \COMMENT{Set the neighborhoods by computing the $k$ NN of $p$, according to the distances $d_{p,q}$}
    \FOR {$\forall p,q:$}
    \IF{$q\in \mathcal{N}_p \wedge p\notin \mathcal{N}_q$}
    \STATE remove $q$ from $\mathcal{N}_p$
    \ENDIF
    \ENDFOR
    \RETURN $\{\mathcal{N}_p\}_{p=1}^N,\{\hat v_p\}_{p=1}^N$
\end{algorithmic}
\end{algorithm}

\subsection{Discrete approximation of the Laplacian}

The approximation of the LB operator over a point cloud is commonly represented as a matrix $L$ of size $N \times N$ where $N$ denotes the number of pixels in each image. If $\phi=\set{\phi_p}_{p=1}^N$ is a vector representing a function $\Phi:\HS\too \Real$ sampled over the normals $\phi_p=\Phi(n_p)$,~ $L \phi$ is the approximation of the LB operator applied to $\Phi$, that is 
\begin{equation}
(L\phi)_p \approx \Delta \Phi (n_p),
\end{equation}
 where $(L\phi)_p$ is the $p^{th}$ coordinate of the vector $L\phi$.

At the core of our algorithm is the approximation of certain eigenfunctions of the Laplace-Beltrami operator over the point cloud of the (unknown) pixels' normals $\set{n_p}$. Since the point cloud of pixels' normals $\set{n_p}$ is an irregular sampling of the hemisphere $\HS$, constructing a reliable approximation to the LB operator over this point cloud is a key step in our algorithm.

Our approximation of the LB operator over the point cloud $\set{n_p}$ is motivated by the following observation (see, e.g., the lemma in \cite{arnold2004lecturesPDE}, p.107): in calculating the LB operator at some point $n_p$ on a sphere, one can replace the sphere by its tangent plane at $n_p$ and calculate the standard Euclidean Laplacian at $n_p$. Thus, in our construction, to approximate the LB at some point $n_p$ we reconstruct locally a small patch of the sphere in a neighborhood $\N_p$ and parameterize it over its tangent plane. We then approximate the Euclidean planar Laplacian over that plane. More specifically, if we denote by $(u,v)$ the coordinates of the tangent plane at $n_p$, then our goal is to find a set of scalar weights $L_{p,q}$, $q\in \N_p$, such that
\begin{equation}\label{e:tp_laplacian_approx}
  \sum_{q\in \N_p} L_{p,q} \phi_q \approx \Phi_{uu}(n_p) + \Phi_{vv}(n_p) = \Delta \Phi(n_p),
\end{equation}
where, as before, $\phi_p=\Phi(n_p)$ is a sampling vector of some smooth function over our hemisphere $\HS$.

Next we describe how to find the local tangent plane approximation at $n_p$, and how to build the weights $L_{p,q}$ to approximate the planar Laplacian over that tangent plane.

\begin{algorithm}
\caption{$LaplacianApproximation(\{\mathcal{N}_p\}_{p=1}^N,\{\hat v_p\}_{p=1}^N)$}
\label{alg:3}
\begin{algorithmic}
  \REQUIRE Neighborhood of each pixel $\mathcal{N}_p$, and normalized intensities vector of each pixel $\hat v_p$.
  
  \FOR {$\forall p = (x,y)$}
  \STATE $\forall q \in \mathcal{N}_p: (n_q^u,n_q^v) \gets PCA_{n\rightarrow 2}(\hat v_q)$ \COMMENT{reduce dimension of $\hat v_q$ from $n$ to $2$ by applying PCA on every $\hat v_{q}$ of $q\in \mathcal{N}_p$ (parameterization over tangent plane)}
  
  \STATE $r_p \gets max_{q\in \mathcal{N}_p}\sqrt{(n_q^u-n_p^u)^2+(n_q^v-n_p^v)^2}$
   
   \IF {$q \notin \mathcal{N}_p$}
   \STATE $w_{p,q} = 0$
   \ENDIF
   \ENDFOR
   \STATE Compute weights $w_{p,q}$ by solving:
   \begin{alignat}{3}
   \min \quad        & \sum_{p,q\in \N_p} w_{p,q}^2 & \nonumber \\
   \text{s.t } \quad & \sum_{q\in \N_p} w_{p,q} (n_q^u - n_p^u) = 0 \quad & \forall{p} \label{eq:weights:2} \\
               \quad & \sum_{q \in \N_p} w_{p,q} (n_q^v - n_p^v) = 0 \quad & \forall{p} \label{eq:weights:3} \\
               \quad &  \sum_{q\in \N_p} w_{p,q} = \frac{1}{r_p^2} \quad & \forall{p} \label{eq:weights:4} \\
               \quad & w_{p,q} \geq 0 & \forall{p,q\in \N_p} \label{eq:weights:5} \\
               \quad & w_{p,q} = w_{q,p} & \forall{p,q\in \N_p} \label{eq:weights:6}            
   \end{alignat}
   \STATE $W\in R^{N\times N} \gets (w_{p,q})$ \COMMENT{$W$ is the matrix whose entries are the weights $w_{p,q}$}
   \STATE $D\in R^{N\times N} \gets diagonal\{d_{11}=\sum_q w_{1,q},...,d_{NN}=\sum_q w_{N,q}\}$
   \STATE $L\gets D-W$ \COMMENT {Discrete laplacian matrix}
   \RETURN L
\end{algorithmic}
\end{algorithm}

\paragraph{Local parameterization over tangent planes}

For each $p$, we would like to build a local tangent plane parameterization to $\HS$ at $n_p$. We will construct this parameterization using the local neighborhood $\N_p$. We project the normalized intensities $\hat{v}_p$ of $\N_p$ onto a two-dimensional plane using Principal Component Analysis (PCA). The coordinates of the projected normals approximate the coordinates on the tangent plane to $\HS$.

\paragraph{LB weight construction.}

Our next objective is to use our planar parameterization to construct the discrete Laplacian operator, $L$. We ignore boundary conditions for now and treat them separately later. Each pixel $p$ is associated with a row in $L$, and the entries of that row represent weights $L_{p,q}, q\in \N_p$, so that a combination of function values $\Phi(n_q)$ with those weights approximates $\Delta \Phi(n_p)$ as described in equation (\ref{e:tp_laplacian_approx}). Our task therefore is to find those weights. Clearly, the weights should vary from one row to the other, since the neighborhoods for each normal $n_p$ can vary. Below we describe a method to determine these sets of weights.

Let us note that our construction of the weights for a single pixel is equivalent to computing Locally Linear Embedding (LLE)~\cite{LLE} weights in the case that the $w_{p,q}$'s are under-determined (number of neighbors 
is greater than $d+1$, where $d$ is the dimension of the points, where in our case $d=2$) while taking the limit of their conditioning parameter to zero. The main 
difference is that our formulation maintains linear precision always  by reproducing the polynomials $1,u,v$ and their linear combinations (equations \ref{eq:weights:2}, \ref{eq:weights:3}, \ref{eq:weights:4}), uses non-negative weights (equation \ref{eq:weights:5}), minimizes the solution's norm, and is global and symmetric (equation \ref{eq:weights:6}). In terms of weights, locally finding the weights for one neighborhood $\N_p$ is the same as using the pseudoinverse in the LLE.

%
%

\subsection{Boundary conditions}
We incorporate both Dirichlet and Neumann boundary conditions along the hemisphere's equator $\partial\HS$. \newtext{Both conditions must be used to solve the embedding problem}. Later on in this section we further explain how we identify points on the equator.

\paragraph{Dirichlet boundary condition.}
For the Dirichlet boundary condition, we only need to calculate weights for interior points $\HS\setminus\partial\HS$. We keep only rows of $L$ corresponding to interior points (not on the equator of the normal hemisphere).
For every equator point \(p\), we replace the corresponding row in $L$ with the standard basis vector that has some value \(t \ne 0\) at position \(p\) and \(0\)'s elsewhere. We denote the new matrix $L_D$. If the eigenvector in the equation $L_D \phi = \lambda \phi$ satisfies $\lambda\ne t$ then necessarily $\phi(p)=0$. In the generic case taking arbitrary $t$ will lead to the desired eigenvectors, and we take \(t=1\). Let us denote by $\wt{n}^z=\set{\wt{n}^z_p}$ the eigenvector of $L_D$ corresponding to the smallest eigenvalue. To enforce this boundary condition we add the following constraints to constraints (\ref{eq:weights:2})-(\ref{eq:weights:6}) in Alg.~\ref{alg:3}.
\begin{alignat}{3}
    \quad & w_{p,q} = 0 \quad & \forall{p \in \partial\HS, \quad p \neq q} \label{eq:Dirichlet:1} \\
    \quad & w_{p,q} = 1 \quad & \forall{p \in \partial\HS, \quad p = q} \label{eq:Dirichlet:2}
\end{alignat}

\renewcommand{\tableScaleSize}{.49}
\begin{figure}[t]
\begin{tabular}{cc}
\includegraphics[width=\tableScaleSize\columnwidth]{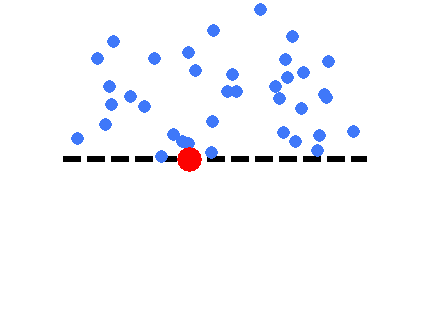}
&\includegraphics[width=\tableScaleSize\columnwidth]{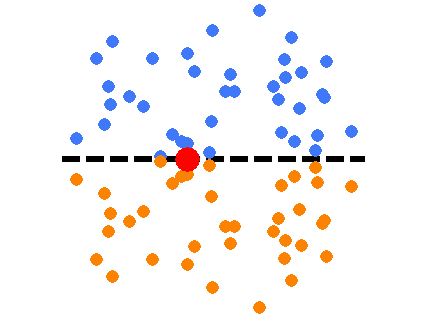}
\end{tabular}
\caption{\textbf{Enforcing Neumann boundary conditions for a pixel:} Left: A pixel on the equator (red) and its local neighborhood (blue) after projection onto 2d. Right: Mirroring of points to enforce Neumann boundary conditions. Since each mirrored pixel attains the value of its original pixel, the derivative in the direction perpendicular to the equator is zero.}

\label{fig:NeumannProcessFig}
\end{figure}

\paragraph{Neumann boundary condition.}
To get Neumann boundary conditions we use symmetric reflection along the domain's boundary. As before, we do not alter the rows of $L$ corresponding to interior points. For an equator pixel $p$, we duplicate its neighborhood $(n_q^u,n_q^v)$, $q\in\N_p$ along the equator via reflection. To get the direction of the equator we use the eigenvector $\wt{n}^z$ calculated with the Dirichlet boundary condition. Since this vector approximates the $z$-coordinate value of the normals its gradient is orthogonal to the boundary (equator). We therefore calculate its approximated gradient (using a linear fit) and rotate it by $\pi/2$ to recover the axis of reflection. We then reflect all the neighbor points $(n_q^u,n_q^v)$, \(q\in\N_p\) in the PCA plane along the equator line we have found, to create corresponding points $q^R$ such that the equator point $(n_p^u,n_p^v)$ lies exactly on the equator line. 
To enforce this boundary condition we add the following constraints to constraints (\ref{eq:weights:2})-(\ref{eq:weights:6}) in Alg.~\ref{alg:3}:
\begin{alignat}{3}
    \quad & w_{p,q^R} = w_{p,q} \quad & \forall{p\in \partial\HS,\quad q, q^R\in \N_p}. \label{eq:Neumann:1}
\end{alignat}

We denote this operator by $L_N$. Note that reflection will force eigenvectors $L_N \phi = \lambda \phi$ to have zero normal derivative across the equator, since we force symmetry w.r.t. the equator. In the kernel of $L_N$ we will have the constant vector. The first two eigenvectors corresponding to non-zero eigenvalues are denoted by $\wt{n}^x,\wt{n}^y$.

\paragraph{Identifying equator points.}

In order to set our boundary conditions, we need to identify the pixels that lie on the equator of the normal hemisphere. If the object we attempt to reconstruct is smooth and convex this set of pixels will include the points on the bounding contour of the silhouette. This assumption has been used by~\cite{SatoICCV2007}. As \cite{Sato2} note, this may not be the case, as is demonstrated in Figure~\ref{fig:BoundaryDiscovery}. To address this issue we introduce a method for identifying the equator that is based on the structure of the manifold itself.

We identify points on the equator by applying the Isomap algorithm to embed the points onto 2 dimensions. The Isomap algorithm uses geodesic distances calculated as the shortest paths between each pair of points. On real images, which exhibit noisy distances, we found that sometimes the embedding projects the hemisphere onto other planes such as the XZ or YZ planes instead of the XY plane. We found it helpful to modify the distances to correspond to a more flattened disc-shaped manifold, to ensure the embedding projects the hemisphere onto the XY. This modification is a local non-linear transformation of the distances $d_{p,q}$ that enlarges large distances: $\tan(\frac{d_{p,q}\pi}{2d_{max}+\epsilon})$. We note that the boundary of the manifold is invariant to this distance transformation.
The resulting planar embedding's boundary corresponds to the boundary of the original manifold. We then select the points on the convex hull, label them as boundary points, remove them from the set, and apply convex hull again. We repeat this until 5\% of points are selected as boundary.
Figure~\ref{fig:BoundaryDiscovery} shows the equator discovered by our method (left). It can be seen that our result closely matches the real equator, whereas the equator from the silhouette  of the object (right) includes many more internal points.

\renewcommand{\tableScaleSize}{3cm}
\begin{figure}[h!]
\center
\includegraphics[height=\tableScaleSize]{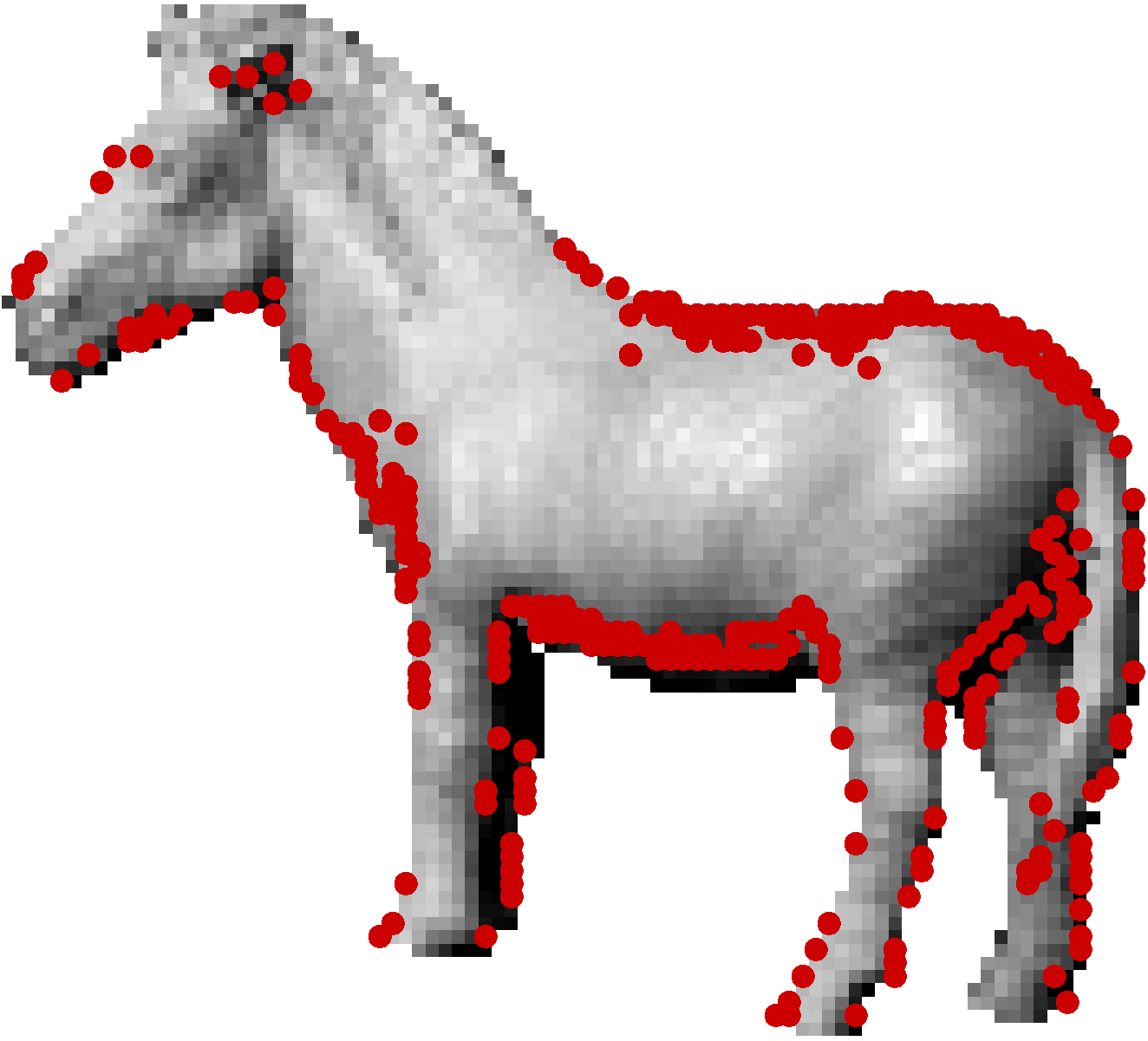}
\includegraphics[height=\tableScaleSize]{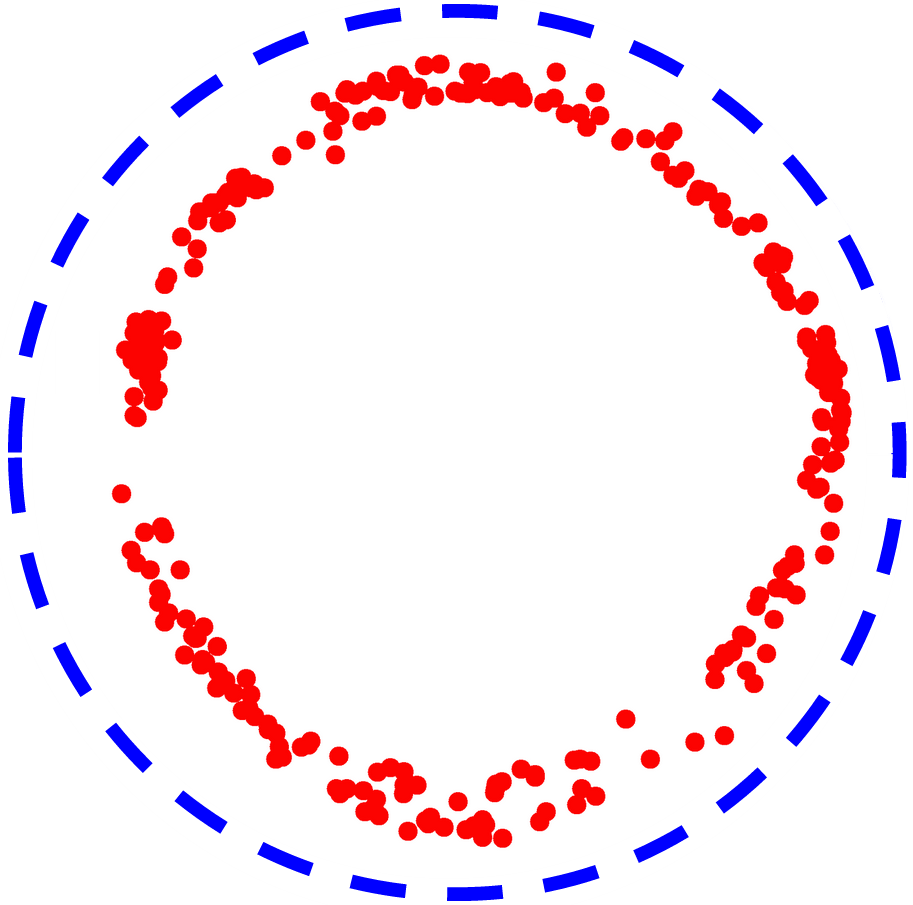}
\vline
\includegraphics[height=\tableScaleSize]{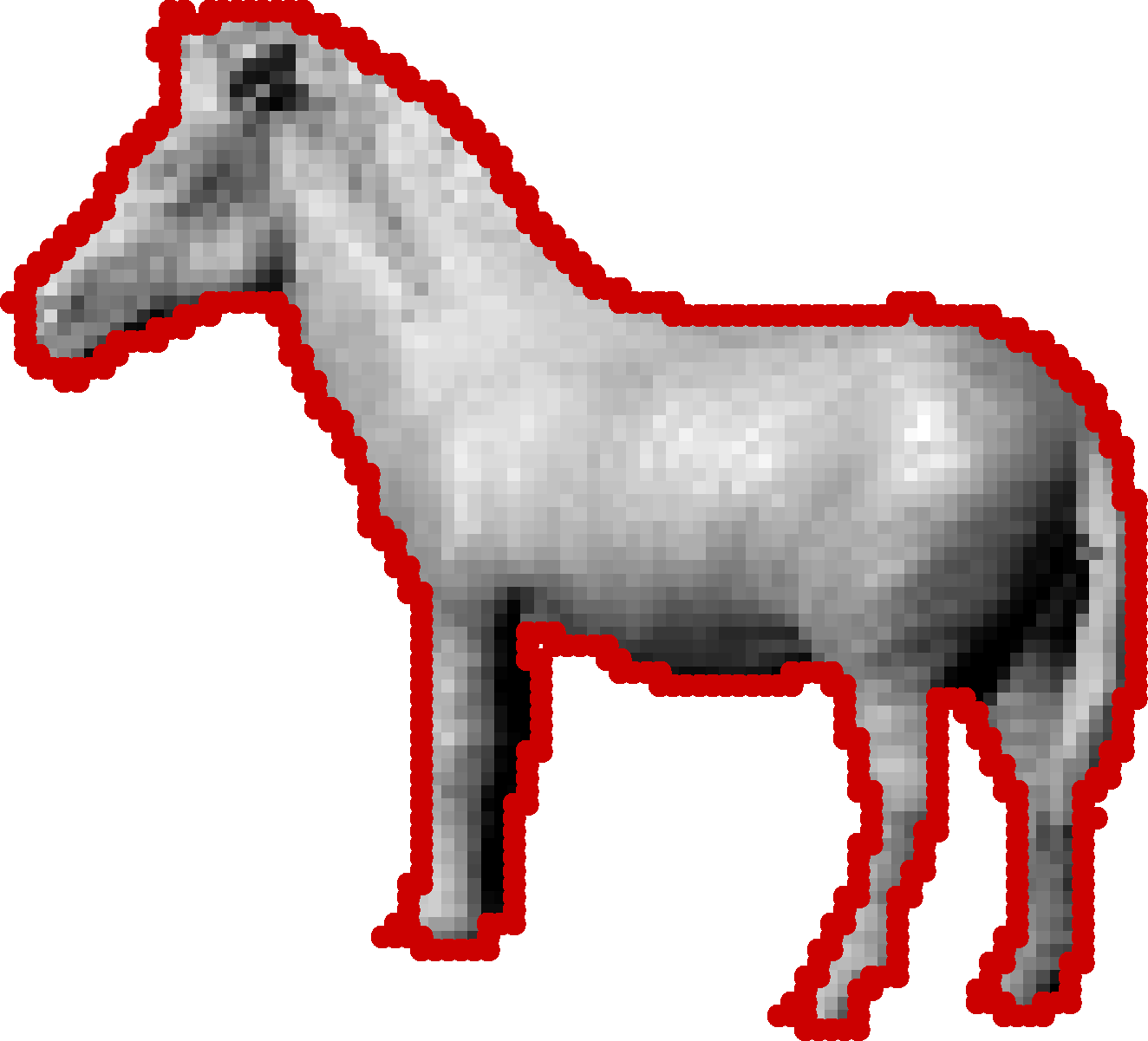}
\includegraphics[height=\tableScaleSize]{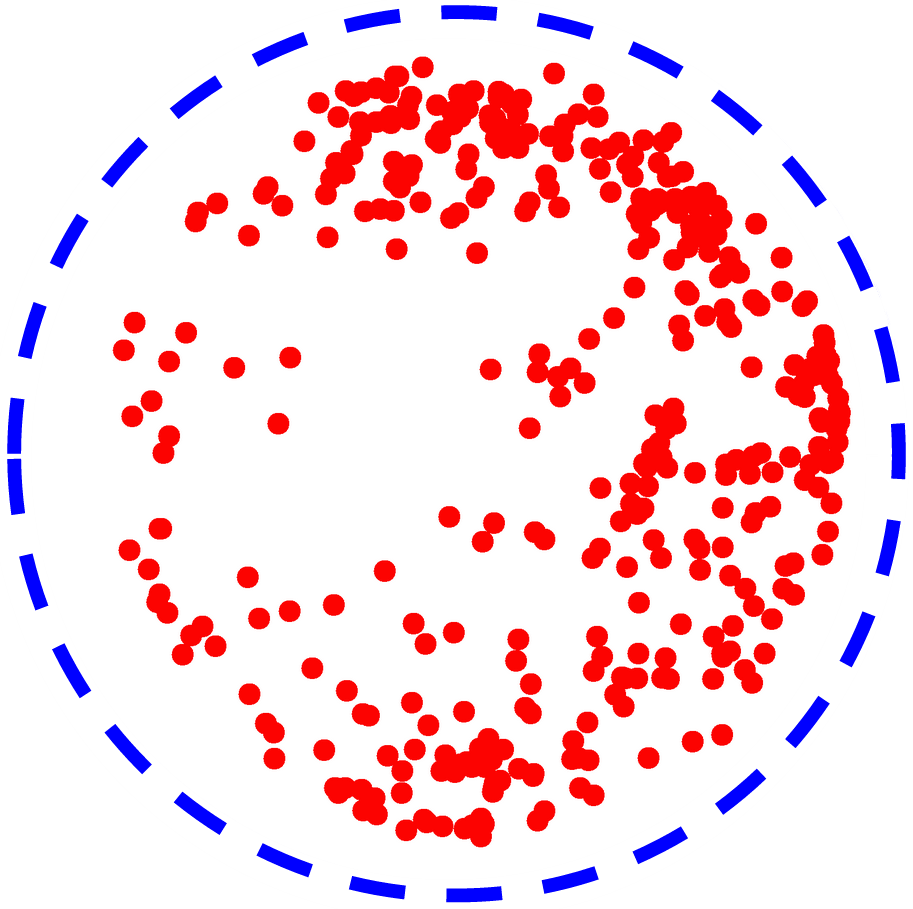}
\caption{\textbf{Comparing equator discovery methods.} Top: Our method. Bottom: Silhouette-based equator discovery. Discovered equator points are marked in red on the image and the normal hemisphere. Observe that silhouette outline pixels (bottom) are scattered over the whole hemisphere. Points discovered by our method (top) more closely match the true equator.}
\label{fig:BoundaryDiscovery}
\end{figure}

We note that the Dirichlet boundary conditions enforce $\wt{n}^z$ to be zero on the equator. A camera cannot actually see pixels whose normal is exactly on the equator since their surface orientation is perpendicular to the viewing direction.  To correct for this, we have shifted the $\wt{n}^z$ values such that the minimum value is set to 0.15.

\begin{algorithm}
\caption{$BoundaryConditions(L)$}
\label{alg:4}
\begin{algorithmic}
  \REQUIRE $L$ is the discrete laplacian matrix of size $N\times N$.
    \STATE Identify the boundary points on the equator
    \STATE $L_D\gets$ apply Dirichlet boundary condition on $L$ 
    \COMMENT{value equals zero} 
    \STATE $L_N\gets$ apply Neumann boundary condition on $L$ 
    \COMMENT{derivative along the $z$ axis equals zero} 
    \RETURN $L_D,L_N$
\end{algorithmic}
\end{algorithm}

\subsection{Reconstructing normals and depth map}

The last stage in our algorithm is getting the approximated normals $\wt{n}_p$ and using them to build a depth map. This is achieved by computing the $\wt{n}^x, \wt{n}^y, \wt{n}^z$ eigenvectors defined above and setting $\wt{n}_p=(\wt{n}_p^x,\wt{n}_p^y,\wt{n}_p^z)$.
As the eigenvectors are found up to scale, we normalize $\wt{n}_p$ to unit length. Next, note that since the two eigenvectors $\wt{n}^x,\wt{n}^y$ correspond to the same eigenvalue, they are subject to arbitrary rotation and reflection. The rotation and reflection are resolved using surface integrability constraints based on~\cite{YuilleIntegrability}. Recovering the shape for a Lambertian object can be done up to a generalized-bas-relief (GBR) transformation, which is a 3 parameter linear transformation of the normals and a non-linear transformation of albedos \cite{BasReliefAmbiguity}. Since we normalize the intensity vectors to factor out albedo, our construction is independent of albedo and the GBR ambiguity is thus restricted to a convex-concave ambiguity. To reconstruct a depth map we apply integrability constraints as in \cite{RonenSH2}, \newtext{for other methods see survey in \cite{durou2016}}. The convex-concave ambiguity is resolved using a global assumption on depth as in \cite{LDR}.

\paragraph{Comparing recovered and expected eigenvalues.}
We know the geometry of the desired manifold is that of a hemisphere. We can therefore compare the recovered eigenvalues of our approximation $L$ to the expected spherical harmonic eigenvalues of $\Delta$ (which are analytically known). If the recovered eigenvalues do not fit the expected eigenvalues, it suggests that our reconstruction will not be very precise. It can be shown that the set of spherical harmonics $Y_{nm}$ can be divided into two sets of functions, $Y_{nm}$ with odd $n+m$, which vanish along the equator, corresponding to the Ditichlet boundary conditions, and $Y_{nm}$ with even $n+m$, whose longitudinal derivatives vanish along the equator, corresponding to the Neumann boundary conditions.

In Figure~\ref{fig:LaplacianEigenvalues-Neumann} we show a comparison between the recovered and expected eigenvalues of our approximation $L$ constructed for the owl object (shown later in Figure~\ref{fig:LambertianNonUniform}) with Dirichlet or Neumann boundary conditions using only the varying illumination images. The figure shows that the 8 smallest eigenvalues computed from $L$ are fairly close to the expected Spherical Harmonic eigenvalues. As noted above, $L\approx \alpha \Delta$, with some constant $\alpha>0$. Therefore, the comparison in Figure~\ref{fig:LaplacianEigenvalues-Neumann} is done after compensating for this scale by fitting a global scale factor to the eigenvalues. Note that multiplying an operator \(L\) by a factor \(\alpha\) scales the eigenvalues but does not alter the eigenvectors.

\renewcommand{\tableScaleSize}{1}
\begin{figure}[h!]
\center
\includegraphics[width=\tableScaleSize\columnwidth]{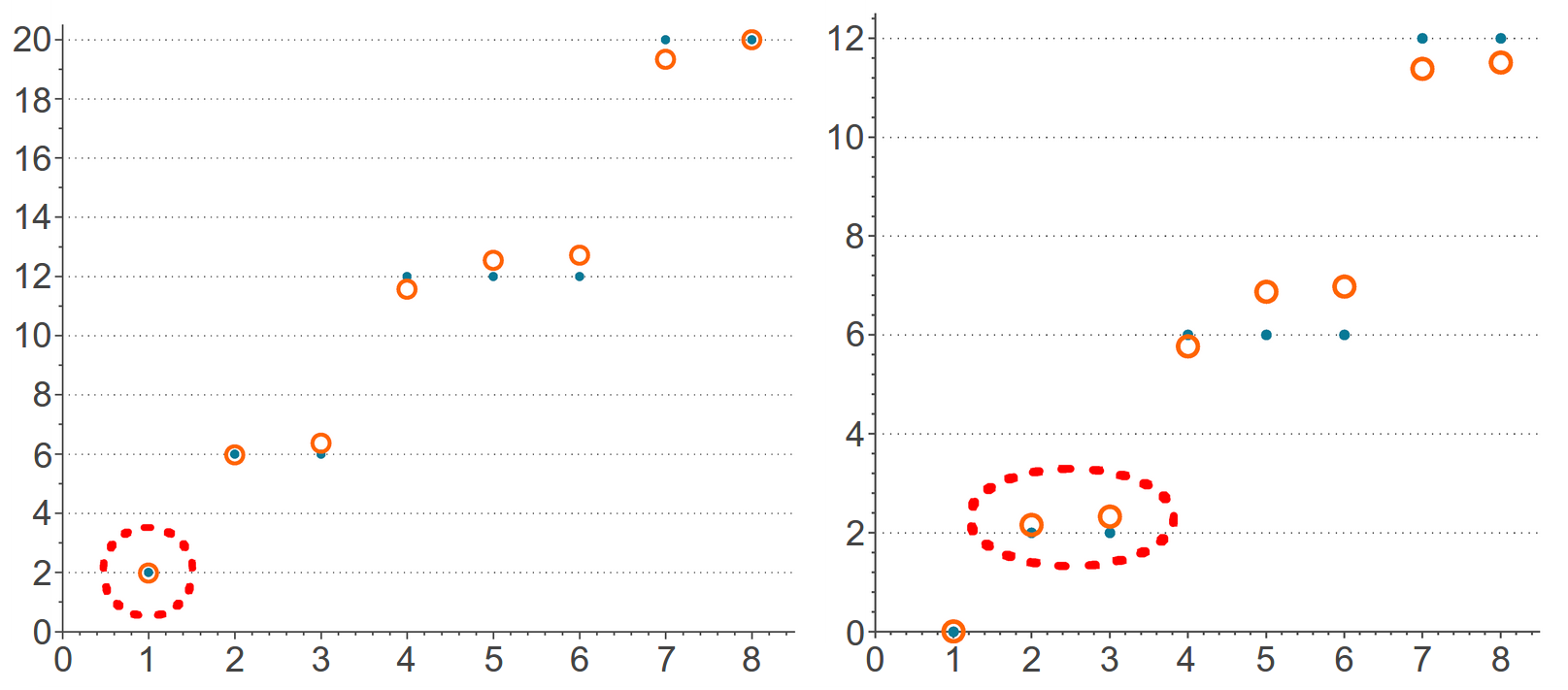}
\caption{\textbf{Comparing recovered and expected eigenvalues}: The expected (blue dots) and approximated (orange circles) scaled eigenvalues of images of the rough ball object for the LB operator with Dirichlet (left) and Neumann (right) eigenvalues. For reconstruction we only use the eigenvectors whose corresponding eigenvalues are marked in red.}
\label{fig:LaplacianEigenvalues-Neumann} \vspace{-0.3cm}
\end{figure}

\begin{algorithm}
\caption{$RecoverDepth(L_D,L_N)$}
\label{alg:5}
\begin{algorithmic}
  \REQUIRE $L_D$ is laplacian matrix with Dirichlet boundary conditions, $L_N$ is the same matrix with Neumann boundary conditions.
        \STATE $\wt{n}^z_p \gets$ first eigenvector of $L_D$
        \STATE $\wt{n}^x_p,\wt{n}^y_p \gets$ second and third eigenvectors of $L_N$
    \STATE rotate $\wt{n}^x_p,\wt{n}^y_p$ to the image axes
    \STATE $z(x,y)\gets depthImage(\wt{n}^x_p,\wt{n}^y_p,\wt{n}^z_p)$
    \RETURN $z(x,y)$
\end{algorithmic}
\end{algorithm}

\section{Experimental results}
We evaluate our method as a general embedding method for a hemisphere manifold and compare it to the following existing embedding methods: Isomap, a modified version of Isomap, Locally Linear Embedding (LLE)~\cite{LLE} and Diffusion maps.
The modified version of Isomap improves its performance for embedding a hemisphere. Isomap embeds a manifold of intrinsic dimension 2 onto 3d by calculating geodesic distances on the manifold,
but then embedding them using multi-dimensional scaling (MDS), which expects these distances to be Euclidean. Since we know the manifold is that of a hemisphere,
we can convert these geodesic distances $d_g$ to Euclidean distances $d_e$ using the following formula, derived from the law of cosines:

\begin{equation}
d_e = \sqrt{2 r^2 (1-\cos(d_g))}
\end{equation}

where $d_g$ is the geodesic distance and where we set the radius $r$ according to the following,

\begin{equation}
r=\frac{d_{max}}{\pi}
\end{equation}

where $d_{max}$ is the longest geodesic distance between any two points. We refer to this modified Isomap based on the hemispherical chordal distances as 'Isomap chordal'.
In this experiment we use 3 models: a sphere, a doll and a rabbit. For each model, we generate 90 images, each lit by a randomly positioned point source.
We repeat the experiment 30 times and show the results in Table~\ref{table:EmbeddingComparison_mean_l2_error}.
In this experiment, the distances are not ideal due to the randomness in the light distribution, which together with the gap between the intrinsic dimension of the hemisphere and the target dimension of the embedding, are known to cause noise
and instability in the LLE and Diffusion maps embeddings \cite{IntrinsicDimensionEstimation}. 
Isomap chordal achieves favorable results to those of Isomap, which suggests that prior knowledge of the manifold shape can significantly improve embedding results.
Unlike Isomap chordal, we do not rely on successful estimation of a global radius parameter $r$ and do not suffer from topological instability \cite{IsomapTopologicalInstability}.
We further add that in this experiment, we have used Procrustes superimposition to optimally fit the resulting embedding onto the unit hemisphere.
In Photometric Stereo this optimal alignment is not available to us. We now conduct a similar experiment to evaluate the different embedding methods for the purpose of Photometric Stereo.
We use the method in \cite{SatoICCV2007} based on the occluding boundary to align the embedded hemispheres for the existing methods on the unit hemisphere and generate surface normals.
Table~\ref{table:EmbeddingComparison_meanAngleError} shows the resulting errors in terms of mean angle error of the normals.

\renewcommand{\tabcolsep}{0.1cm}
\begin{table}[htbp]
  \centering
  \caption{\textbf{Various objects, 90 images, random single point source}: The table shows the mean point pair distance after optimal Procrustes superimposition to the unit hemisphere.}
  \begin{tabular}{rrccccc}
Model & \multicolumn{1}{l|}{\textbf{Error}} & \specialcell{Our\\method} & Isomap & \specialcell{Isomap\\chordal} & LLE & \specialcell{Diffusion\\Map}\\
\hline
				&	\multicolumn{1}{l|}{Mean}&	\textbf{0.06}&	0.1&	0.08&	0.32&	0.3\\
\multicolumn{1}{l}{Sphere}	&	\multicolumn{1}{l|}{Median}&	\textbf{0.06}&	0.1&	0.08&	0.32&	0.3\\
				&	\multicolumn{1}{l|}{Std}&	0.006&	0.008&	0.009&	0.054&	0.04\\
\hline
				&	\multicolumn{1}{l|}{Mean}&	\textbf{0.08}&	0.12&	0.09&	0.26&	0.24\\
\multicolumn{1}{l}{Rabbit}	&	\multicolumn{1}{l|}{Median}&	\textbf{0.08}&	0.12&	0.09&	0.22&	0.22\\
				&	\multicolumn{1}{l|}{Std}&	0.006&	0.008&	0.008&	0.107&	0.058\\
\hline
				&	\multicolumn{1}{l|}{Mean}&	\textbf{0.08}&	0.13&	0.1&	0.25&	0.17\\
\multicolumn{1}{l}{Doll}	&	\multicolumn{1}{l|}{Median}&	\textbf{0.07}&	0.13&	0.1&	0.25&	0.17\\
				&	\multicolumn{1}{l|}{Std}&	0.008&	0.007&	0.006&	0.046&	0.027\\

\end{tabular}%
  \label{table:EmbeddingComparison_mean_l2_error}%
\end{table}%

\renewcommand{\tabcolsep}{0.1cm}
\begin{table}[htbp]
  \centering
  \caption{\textbf{Various objects, 90 images, random single point source}: The table shows the mean normal angle error.}
  \begin{tabular}{rrccccc}
Model & \multicolumn{1}{l|}{\textbf{Error}} & \specialcell{Our\\method} & Isomap & \specialcell{Isomap\\chordal} & LLE & \specialcell{Diffusion\\Map}\\
\hline
				&	\multicolumn{1}{l|}{Mean}&	\textbf{11.95}&	17.43&	13.2&	37.21&	28.73\\
\multicolumn{1}{l}{Sphere}	&	\multicolumn{1}{l|}{Median}&	\textbf{11.69}&	17.48&	13.0&	39.83&	28.23\\
				&	\multicolumn{1}{l|}{Std}&	1.457&	0.902&	1.064&	11.072&	6.996\\
\hline
				&	\multicolumn{1}{l|}{Mean}&	\textbf{11.16}&	18.26&	14.82&	27.08&	23.35\\
\multicolumn{1}{l}{Rabbit}	&	\multicolumn{1}{l|}{Median}&	\textbf{10.47}&	18.32&	14.91&	26.47&	20.59\\
				&	\multicolumn{1}{l|}{Std}&	3.662&	1.027&	1.138&	11.201&	9.18\\
\hline
				&	\multicolumn{1}{l|}{Mean}&	\textbf{13.42}&	26.44&	23.04&	36.79&	23.1\\
\multicolumn{1}{l}{Doll}	&	\multicolumn{1}{l|}{Median}&	\textbf{13.25}&	26.45&	23.08&	41.64&	23.8\\
				&	\multicolumn{1}{l|}{Std}&	2.449&	2.193&	2.144&	11.12&	2.683\\

\end{tabular}%
  \label{table:EmbeddingComparison_meanAngleError}%
\end{table}%

To test the performance of the proposed algorithm on real images, we have conducted experiments with a variety of materials and lighting conditions. In all experiments the images were resized to 110 pixels in width (and 100 to 170 pixels in height) for faster computation times. On these images our algorithm takes between 5 to 20 minutes on a quad-core 2.8GHz PC, where the time is mainly dominated by the number of pixels and considerably less by the number of input images. We note that our code is implemented in Matlab and can be further optimized, especially the per-pixel computations which can be massively parallelized.

In all experiments we set the number of neighbors \(k\) to 5\% of pixels. We identify outlier pixels by removing pixels that are not favored by the neighbors (a fraction of their neighbors did not choose them as neighbors, using 80\% as a threshold) as we anticipate such pixels to be corrupt or subject to significant noise. We remove these pixels from the computation, and then set their recovered normals by interpolation from neighboring pixels in the image. Instead of using the radius $r$ in equation (\ref{eq:weights:4}) the constant value of 1 was used, as it was found to produce more stable results in practice.

Where appropriate, we compare our results to those of Sato et al.~\cite{Sato2} \color{black} and to Favaro and Papadhimitri \cite{LDR}\color{black}. The code of both methods was obtained from the authors. As an evaluation error measure we use the mean degree error of the normals.

We begin by showing results on objects lit by uniformly distributed point sources. We use the benchmark of~\cite{ConsensusPS} which consists of 4 objects: a terracotta warrior, a doll and a relief sculpture with Lambertian reflectance and a rough non-Lambertian ball. The data set includes between 46 and 57 images per object. In lack of ground truth, we used the normals recovered from their calibrated PS method as ground truth.  Figure~\ref{fig:MixedReflectanceUniform} and Table~\ref{table:MixedReflectanceUniform} show the results of this experiment. This dataset consists of many images which exhibit significant attached shadows and therefore demonstrates the robustness of our method to this phenomena. We note that the relief sculpture exhibits significant cast shadows, which are not modeled by our method, and can account for the higher error measurements for both our method and~\cite{Sato2}.

In the next experiment, the results of which are shown in Figure~\ref{fig:LambertianNonUniform} and Table~\ref{table:LambertianNonUniform}, we use the benchmark of~\cite{LDR}. The data set includes 7 objects and 5 to 12 images per object obtained with point source lightings. Lighting directions are close to the viewing directions with a mean angle of \(23^{\circ}\) and standard deviation of \(10^{\circ}\). We compare our reconstruction to `ground truth' shapes produced in~\cite{LDR} by applying a calibrated Photometric Stereo algorithm that assumes Lambertian reflectance to the images, which are available with the dataset. Lights in this experiment cover only frontal directions, which can account for higher errors compared to the previous data set. \cite{Sato2} assume uniformly distributed point source lights and estimate an essential parameter based on this assumption. This experiment does not adhere to their assumptions and we have not been fruitful in applying their method to these settings. 

Finally, in Figure~\ref{fig:NonLambertianAmbient} and Table~\ref{table:NonLambertianAmbient} we show results obtained with non-Lambertian objects and natural illumination which consists of diffuse lighting coming from several directions as well as light reflected from walls and other objects in the room. The objects, initially diffuse, were laser scanned to produce ground truth shapes. They were then sprayed with a thin shiny coating to make them specular and pictured under ambient illumination that includes reflectance from surrounding walls and objects and is roughly uniform. We use between 25 and 31 images per model. Although we do have scanned shapes of the objects, very precise registration is required to properly report quantitative error measurements, since even a small pixel shift can cause the normals to change dramatically. In lack thereof, we use a different error estimation method. For each pixel normal, we take the 9x9 patch in which it lies, and select the closest normal from the 3d scan that lies 
in that patch. We apply the same procedure for our method, Sato and Favaro. Note that it appears that Favaro achieves best results on all lambertian objects and not winning in all non-lambertian ones.

\renewcommand{\tabcolsep}{0.1cm}
    \begin{table}[htbp]
  \centering
  \caption{\textbf{Uniformly distributed distant point sources}: The table shows the mean normal angle errors obtained with each algorithm. In this scenario, our method is comparable to Sato.}
\begin{tabular}{rccc}
\multicolumn{1}{l|}{\textbf{Experiment set}} & \textbf{\specialcell{Our\\method}} & Sato & Favaro \\
\hline
\multicolumn{1}{l|}{Terracotta warrior} & 12.7  & 12.9 & \textbf{11.7} \\
\multicolumn{1}{l|}{Doll} & \textbf{10.6}  & 12.6 & 11.8 \\
\multicolumn{1}{l|}{Rough ball} & \textbf{5.9}  & 16.0 & 43.2 \\
\multicolumn{1}{l|}{Relief sculpture} & 22.3  & 25.7 & \textbf{9.5} \\

\end{tabular}%
  \label{table:MixedReflectanceUniform}%
\end{table}%

\renewcommand{\tableScaleSize}{0.175}
\begin{figure}[ht]
\centering
\begin{tabular}{cccc}
\import{Sources/res/}{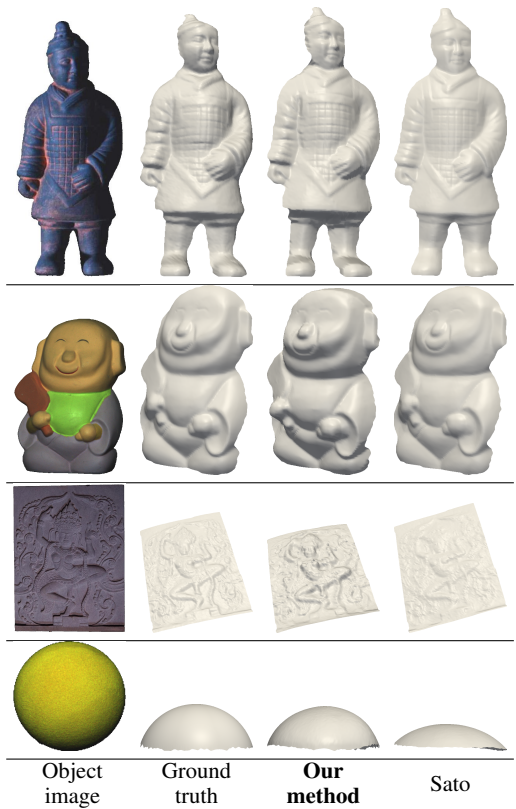}
\specialcell{Object\\image}&\specialcell{Ground\\truth}&\textbf{\specialcell{Our\\method}}&Sato\\
\hline
\end{tabular}
\caption{Objects lit by \newtext{uniformly distributed directional light sources.} The top 3 objects are Lambertian and the bottom object is a rough, non-Lambertian ball. As light comes from many directions, many of the input images exhibit significant attached shadows. These reconstructions demonstrate the robustness of our method to attached shadows.}
\label{fig:MixedReflectanceUniform}
\end{figure}

\renewcommand{\tabcolsep}{0.1cm}
\begin{table}[htbp]
  \centering
  \caption{\textbf{Lambertian objects, non-uniform distant point sources}: The table shows the mean normal angle errors obtained with each algorithm.}
\begin{tabular}{rccc}
\multicolumn{1}{l|}{\textbf{Experiment set}} & \textbf{\specialcell{Our\\method}} & Sato & Favaro\\
\hline
\multicolumn{1}{l|}{Redfish} & 14.0  & 44.3 & \textbf{5.6} \\
\multicolumn{1}{l|}{Octopus} & 19.1  & 30.2 & \textbf{6.64} \\
\multicolumn{1}{l|}{Rock} & 16.4 & 22.1 & \textbf{11.61} \\
\multicolumn{1}{l|}{Horse} & 9.6 & 23.9 & \textbf{4.8} \\
\multicolumn{1}{l|}{Owl} & 10.4 & 24.2 & \textbf{6.63} \\
\multicolumn{1}{l|}{Cat} & 14.6 & 26.6 & \textbf{5.37} \\
\multicolumn{1}{l|}{Buddha} & 11.4  & 22.8 & \textbf{4.98} \\

\end{tabular}%
  \label{table:LambertianNonUniform}%
\end{table}%

\renewcommand{\tableScaleSize}{0.175}
\begin{figure}[ht]
\centering
\begin{tabular}{cccc}
\import{}{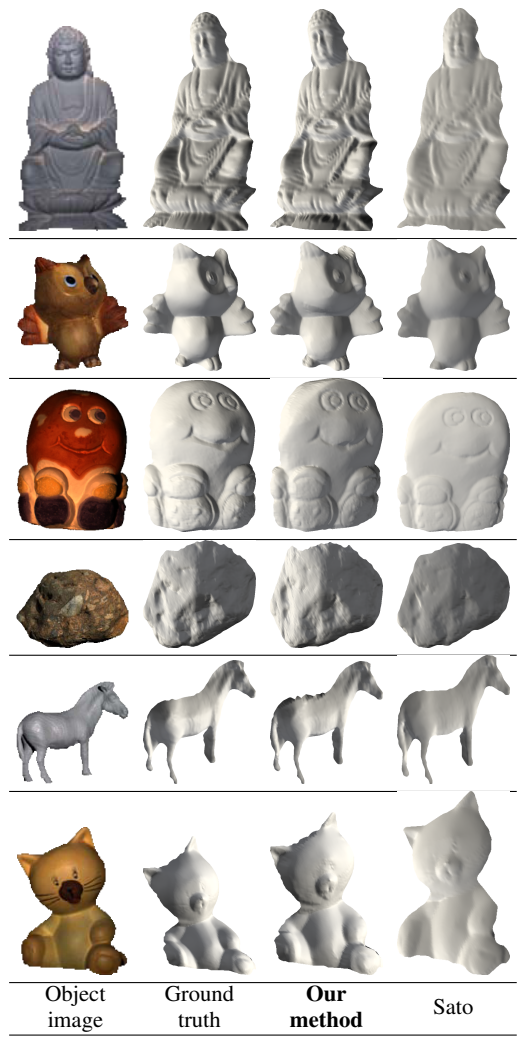}
\specialcell{Object\\image}&\specialcell{Ground\\truth}&\textbf{\specialcell{Our\\method}}&Sato\\
\hline
\end{tabular}
\caption{Lambertian objects lit by distant point sources.}
\label{fig:LambertianNonUniform}
\end{figure}

\renewcommand{\tabcolsep}{0.1cm}
\begin{table}[htbp]
  \centering
  \caption{\textbf{Non-Lambertian objects, ambient lighting}: The table shows the mean normal angle errors obtained with each algorithm.}
\begin{tabular}{rccc}
\multicolumn{1}{l|}{\textbf{Experiment set}} & \textbf{\specialcell{Our\\method}} & Sato & Favaro \\
\hline
\multicolumn{1}{l|}{Duck} & \textbf{7.3}  & 27.4 & 77.6 \\
\multicolumn{1}{l|}{Rabbit} & \textbf{8.1}  & 23.4 & 74.6 \\
\multicolumn{1}{l|}{Rooster} & \textbf{10.4}  & 20 & 77.4 \\
\multicolumn{1}{l|}{Onion} & \textbf{7.5}  & 28.8 & 15.02 \\
\multicolumn{1}{l|}{Pineapple} & \textbf{10.5}  & 15.5 & 27.5 \\

\end{tabular}%
  \label{table:NonLambertianAmbient}%
\end{table}%

\renewcommand{\tableScaleSize}{0.175}
\begin{figure}[ht]
\centering
\begin{tabular}{cccc}
\import{}{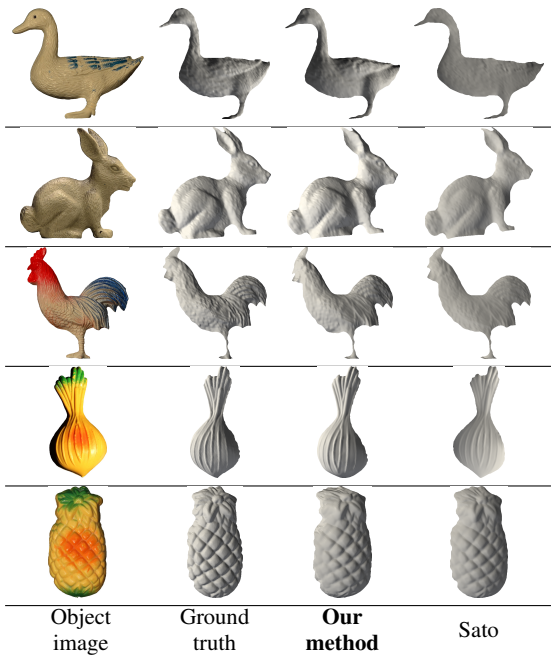}
\specialcell{Object\\image}&\specialcell{Ground\\truth}&\textbf{\specialcell{Our\\method}}&Sato\\
\hline
\end{tabular}
\caption{Non-Lambertian objects lit by ambient illumination.}
\label{fig:NonLambertianAmbient}
\end{figure}

\section{Conclusion}
We have presented an algorithm for Photometric Stereo that uses a generic approach in which intensity measurement vectors are embedded into the unit hemisphere of forward-facing surface normals. We have shown that the algorithm successfully reconstructs objects under quite general lighting conditions and reflectance properties that do not need to be known in advance.
Further work could involve developing better distance measures for finding pixel neighborhoods. For a better approximation of the Laplacian, local geometric analysis at a pixel can be extended beyond a tangent plane to take into account the local geometric curvature. \no{The steps in our algorithm , such as \cite{Frankot88amethod}, could be combined to a unified optimization problem}. Finally, the algorithm can be further used in other scenarios that require hemispherical embedding.

\section{Appendix}
Our proofs rely on a spherical harmonic representation of reflectance. Below we introduce notation, which for the most part follows the notation of~\cite{RonenSH1}. Spherical harmonics are orthogonal functions from the unit sphere to the complex plain $Y_{nm}:{\cal S}^2 \rightarrow {\cal C}$, $n=0,1,2...$ and $-n \le m \le n$, defined as
\begin{equation}
Y_{nm}(\theta,\phi) = \sqrt{\frac{(2n+1)}{4\pi} \frac{(n-|m|)!}{(n+|m|)!}} P_{n|m|}(\cos\theta) e^{im\phi},
\end{equation}
where $P_{nm}$ are the {\em associated Legendre functions}, defined as
\begin{equation}  \label{eq:assoc_legendres}
P_{nm}(z) = \frac{(1-z^2)^{m/2}}{2^n n!} \frac{d^{n+m}}{dz^{n+m}}(z^2-1)^n. \end{equation}
Each harmonic function $Y_{nm}$ is a polynomial of degree $n$ of the variables $(x,y,z)$ on the unit sphere. We will therefore interchangeably use $Y(\theta,\phi)$ and $Y(p)$ with $p=(x,y,z)^T$ and $x=\sin\theta\cos\phi$, $y=\sin\theta\sin\phi$, and $z=\cos\theta$.

To work with real functions we use the notation $Y^e_{nm}=\mathrm{Re}(Y_{nm})$ and $Y^o_{nm}=\mathrm{Im}(Y_{nm})$ ($1 \le m \le n$), which denote the even and odd components of the harmonics. Note that the zonal harmonics ($Y_{n0}$) are real, and that the real spherical harmonics too are orthogonal. In addition, to simplify notation, we interchangeably use a single subscript notation $Y_s$ with $s=1,2,3,...$. For this notation we order the real harmonics by their order $n$ and within every order by $m$, placing each even harmonic before its corresponding odd one, i.e., $Y_{00},Y_{10},Y^e_{11},Y^o_{11},Y_{20},...$.

We will be interested in surfaces whose reflectance kernel is isotropic and band limited. An example for such kernels is the Lambertian reflectance kernel $k(\theta)=\max(\cos\theta,0)$, which can be approximated to a great accuracy with the first nine harmonics ($n \le 2$). As is shown in \cite{RonenSH1}, the reflectance function for such surfaces can be expressed as a convolution of the lighting function with the reflectance kernel. As a consequence, the intensity $I_p$ at a pixel $p$ with surface normal $n_p$ and albedo $\rho_p$ can be expressed in the following form. Let $\ell(\theta,\phi)$ denote the environment lighting as a function of direction, and let $\ell(\theta,\phi) = \sum_{s=1}^\infty \ell_s Y_s(\theta,\phi)$ be its harmonic decomposition, then
\begin{equation}  \label{eq:intensity}
I_p = \rho_p \sum_{s=1}^\infty \alpha_s \ell_s k_s Y_s(n_p),
\end{equation}
where $\alpha_s$ are constants due to the Funk-Hecke (convolution) theorem, $k(\theta)=\sum_{n=0}^\infty k_{(n)} Y_{n0}(\theta,\phi)$ is the harmonic expansion of the kernel and $k_{s(n,m)}=k_{(n)}$.

For Lambertian surfaces we are interested in harmonics up to order 2. The kernel coefficients and the Funk-Hecke constants for these orders are
\begin{align}  \label{eq:k_lambert}
k_{(0)}&=\frac{\sqrt{\pi}}{2}   & \alpha_{(0)}&=\pi \nonumber \\
k_{(1)}&=\sqrt{\frac{\pi}{3}}   & \alpha_{(1)}&=\frac{2\pi}{3} \\
k_{(2)}&=\frac{\sqrt{5\pi}}{8} & \alpha_{(2)}&=\frac{\pi}{4}, \nonumber
\end{align}
and the harmonics are
\begin{align}
Y_1 = Y_{00}   &=\frac{1}{\sqrt{4\pi}} &
Y_2 = Y_{10}   &=\sqrt{\frac{3}{4\pi}}z \nonumber \\
Y_3 = Y_{11}^e &=\sqrt{\frac{3}{4\pi}}x &
Y_4 = Y_{11}^o &=\sqrt{\frac{3}{4\pi}}y \nonumber \\
Y_5 = Y_{20}   &=\frac{1}{2}\sqrt{\frac{5}{4\pi}}(3z^2-1) &
Y_6 = Y_{21}^e &=3\sqrt{\frac{5}{12\pi}}xz \label{eq:9harmonics} \\
Y_7 = Y_{21}^o &=3\sqrt{\frac{5}{12\pi}}yz &
Y_8 = Y_{22}^e &=\frac{3}{2}\sqrt{\frac{5}{12\pi}}(x^2-y^2) \nonumber \\
Y_9 = Y_{22}^o &=3\sqrt{\frac{5}{12\pi}}xy. \nonumber & &
\end{align}

%




\subsection{Equivalence of distances for isotropic reflectance kernels}

\begin{claim}{}\label{sh_isotropic_reflectance_proof_appendix}
Under uniformly distributed \newtext{directional} light sources and an isotropic, band limited reflectance kernel, the $\ell_2$ distances of normalized intensity vectors $\hat v_p$ and $\hat v_q$ of points $p$ and $q$ with \textit{nearby} normals $n_p$ and $n_q$ are proportional to the spherical distance $d(n_p,n_q)$ between the corresponding surface normals, i.e.,
$$\|\hat v_p - \hat v_q\|_2 = c \cdot d(n_p,n_q),$$
where the constant $c$ depends on the reflectance kernel.
\end{claim}
\begin{proof}{}
We prove the claim by evaluating the expression:
$$ \|\hat v_p - \hat v_q\|_2^2 =
2 - 2\hat v_p^T\hat v_q. $$
The (un-normalized) vector \(v_{p}\) contains the intensities observed at pixel \(p\) over a set of images \(k=1..K\), each lit by a point light source from direction \(l_k \in {\cal S}^2\) whose magnitudes are equal.
Using \eqref{eq:intensity}, the intensity $I_{kp}$ of $p$ can be written as,
$$ I_{kp}=\rho_p \sum_{s=1}^\infty \ell_{ks} c_s Y_s(n_p), $$
where \(\rho_p\) and \(n_p\) respectively are the albedo and normal at $p$, \(\ell_{ks}\) are the harmonic coefficients of the light in the $k$'th image, and \(c_s=\alpha_s k_s\) are constants whose values depend on the reflectance kernel. As we will be interested in normalized intensity vectors $\hat v_p$, we can assume w.l.o.g. that \(\rho_p=1\) and likewise that the lights are of unit magnitudes .

The inner product of two intensity vectors \(v_{p}\) and \(v_{q}\) is given by,
\[ v_{p}^Tv_{q}= \sum_{k}^K \left ( \sum_{s=1}^\infty \ell_{ks} c_s Y_s(n_p)\right) \left (\sum_{t=1}^\infty \ell_{kt} c_t Y_t(n_q)\right) = \sum_{s=1}^\infty \sum_{t=1}^\infty N_{st} L_{st},\]
where,
$$L_{st}=\sum_{k}^K \ell_{ks} \ell_{kt}$$
$$N_{st}=c_s k_s Y_s(n_i) c_t k_t Y_t(n_j).$$

Under the assumption that the light in the $K$ images is distributed uniformly $L_{st}=\mathbbm{1}_{s=t}$. This is because for a point source light at direction \(l_{k}\), which is expressed as a Dirac delta function \(\delta_{l_{k}}\), \(\ell_{ks}=<\delta_{l_{k}},Y_s>=Y_s(l_{k})\), so that,
$$L_{st}=\sum_{k}^K Y_s(l_{k}) Y_t(l_{k}),$$
and for uniform light of unit intensity, we get
$$L_{st} \approx \int_{S^2}Y_s(l)  Y_t(l) dl=\mathbbm{1}_{s=t}$$
due to the orthonormality of the spherical harmonics. The inner product $v_{p}^Tv_{q}$  therefore simplifies to
$$v_{p}^Tv_{q} = \sum_{s=1}^\infty N_{s}, $$
where \(N_s=c^2_s Y_s(n_p) Y_s(n_q)\).

We notice next that $N_s$ can be expressed as a univariate polynomial in $z=\cos\theta$, where $\theta$ is the angle between $n_p$ and $n_q$. As the inner product $v_{p}^Tv_{q}$ and the spherical harmonics are invariant to a global rotation of the normals, we can orient our coordinate system so that \(n_p=[0,0,1]^T\) and \(n_q=[\sin \theta, 0, \cos \theta]^T\) (and \(\phi=0\)). In this coordinate frame $Y_s(n_p)$ is constant and $Y_s(n_q)$ is a (scaled) associated Legendre function, which is polynomial of degree $n$ in $z=\cos\theta$. We can therefore write
$$N_s = \sum_{r=0}^n a_{sr}z^r$$
with coefficients $a_{sr}$ that depend on the reflectance kernel and the harmonic order, $s$. Since $n_p$ and $n_q$ are nearby, we use the Taylor approximation \(\cos\theta=1-\frac{\theta^2}{2}+O(\theta^4)\), which yields,
$$N_s = \sum_{r=0}^n a_{sr}(1-\frac{\theta^2}{2})^r + O(\theta^4) = a_{s}-b_{s}\frac{\theta^2}{2}+O(\theta^4),$$
where $a_s=\sum_{r=0}^n a_{sr}$ and $b_s=\sum_{r=1}^n r a_{sr}$. We now have,
$$v_{p}^Tv_{q} = \sum_{s=1}^\infty \left(a_{s} - b_{s}\frac{\theta^2}{2} + O(\theta^4) \right)$$
We further assume that the reflectance kernel is band limited, so that harmonic terms of orders $n>N$ for a finite $N$ can be omitted. Therefore,
$$v_{p}^Tv_{q} = \sum_{s=1}^S \left(a_{s} - b_{s}\frac{\theta^2}{2}\right) + O(N^2\theta^4),$$
where $S=(N+1)^2$. We further denote $a=\sum_{s=1}^S a_s$ and $b=\sum_{s=1}^S b_s$.

$v_{p}^Tv_{p}$ can be computed simply by plugging $\theta=0$ in the previous expression, obtaining $v_{p}^Tv_{p} = a$, and $v_q^Tv_q=v_p^Tv_p$ due to rotation invariance.

We can now evaluate the inner product between the normalized intensity vectors,
$$\hat v_{p}^T \hat v_{q} = \frac{v_{p}^Tv_{q}}{\|v_{p}\| \|v_{q}\|} =
\frac{a - b \frac{\theta^2}{2}}{a}+O(N^2\theta^4) = $$
$$1-\left(\frac{(c\theta)^2}{2}\right)+O(N^2\theta^4),$$
where $c=\sqrt{b/a}$. Note that $a$ is positive since $a=\|v_p\|^2$ and $b$ is positive for sufficiently small $\theta$, since $v_p^T v_q<a$. We conclude that, up to $O(N^2\theta^4)$ terms,
$$\|\hat v_p - \hat v_q\|_2 = \sqrt{2 - 2\hat v_p^T\hat v_q} = c\theta = c \cdot d(n_p,n_q)$$

\end{proof}


%



\subsection{Intensity distances to geodesic distances for the Lambertian reflectance kernel}

\begin{claim}{}\label{sh_isotropic_reflectance_proof_appendix_Lapbertian}
The reflectance-dependent constant $c$ can be computed analytically for Lambertian reflectance, and is approximately~0.93.
\end{claim}
\begin{proof}{}
As proved in the previous claim,
$$\|\hat v_p - \hat v_q\|_2 = c \cdot d(n_p,n_q).$$
We now wish to calculate the constant $c$ for the Lambertian kernel. To this end we assume $n_p=(0,0,1)^T$, $n_q=(\sin\theta,0,\cos\theta)^T$ and compute $N_{s}=\alpha^2_{(n)} k^2_{(n)} Y_s(n_p) Y_s(n_q)$ as a function of $\theta$.
Using~\eqref{eq:k_lambert}-\eqref{eq:9harmonics} we get
\begin{align*}
N_{1}&=\frac{\pi}{4} \\
N_{2}&=\frac{\pi}{3}\cos \theta \\
N_{5}&=\frac{5 \pi}{128}(3 \cos^2 \theta-1),
\end{align*}
and the rest of the terms vanish. Therefore,
$$ v_p^T v_q=\sum_{s=1}^9 N_s = \frac{\pi}{4}+\frac{\pi}{3}\cos \theta+\frac{5 \pi}{128}(3 \cos^2 \theta-1). $$
Replacing $\cos\theta$ by its Taylor approximation, $\cos \theta \simeq 1-\frac{\theta^2}{2}$, we get that,
$$
v_p^T v_q\simeq
\frac{127\pi}{192}-\frac{109\pi}{384}\theta^2$$
and $v_p^Tv_p=v_q^Tv_q$ are obtained by plugging $\theta=0$.
Denote, $a=127\pi/192$ and $b=109\pi/192$, we can now calculate the inner product of the normalized observation vectors $\hat v_p^T \hat v_q$,
$$\hat v_p^T \hat v_q = \frac{v_p^T v_q}{\|v_p\| \|v_q\|} \simeq
\frac{a-b\frac{\theta^2}{2}}{a} =
1-\frac{b}{a}\frac{\theta^2}{2}.$$
Let \(c=\sqrt{\frac{b}{a}}
\approx 0.93\).
We obtain,
$$\|\hat v_p - \hat v_q\|_2
\approx 0.93\theta.$$
\end{proof}





{\small \vspace{-0.1cm}
\bibliographystyle{spmpsci}
\bibliography{myrefs}

\begin{thebibliography}{10}
\providecommand{\url}[1]{{#1}}
\providecommand{\urlprefix}{URL }
\expandafter\ifx\csname urlstyle\endcsname\relax
  \providecommand{\doi}[1]{DOI~\discretionary{}{}{}#1}\else
  \providecommand{\doi}{DOI~\discretionary{}{}{}\begingroup
  \urlstyle{rm}\Url}\fi

\bibitem{HeliometricStereo}
Abrams, A., Hawley, C., Pless, R.: Heliometric stereo: Shape from sun position.
\newblock In: CVPR (2012)

\bibitem{WebcamPS}
Ackermann, J., Langguth, F., Fuhrmann, S., Goesele, M., Darmstadt, T.:
  Photometric stereo for outdoor webcams.
\newblock In: ECCV (2012)

\bibitem{EntropyMinimizationPS}
Alldrin, N., Mallick, S., Kriegman, D.: Resolving the generalized bas-relief
  ambiguity by entropy minimization.
\newblock In: CVPR (2007)

\bibitem{arnold2004lecturesPDE}
Arnold, V., Cooke, R.: Lectures on Partial Differential Equations.
\newblock Universitext (1979). Springer (2004).
\newblock \urlprefix\url{http://books.google.co.il/books?id=qlNJAYwmfTcC}

\bibitem{IsomapTopologicalInstability}
Balasubramanian, M., Schwartz, E.L.: The isomap algorithm and topological
  stability.
\newblock Science  (2002)

\bibitem{RonenSH2}
Basri, R., Jacobs, D., Kemelmacher, I.: Photometric stereo with general,
  unknown lighting.
\newblock IJCV  (2007)

\bibitem{RonenSH1}
Basri, R., Jacobs, D.W.: Lambertian reflectance and linear subspaces.
\newblock IEEE Trans. Pattern Anal. Mach. Intell. \textbf{25}(2), 218--233
  (2003)

\bibitem{BasReliefAmbiguity}
Belhumeur, P.N., Kriegman, D.J., Yuille, A.L.: {The bas-relief ambiguity}.
\newblock IJCV  (1999)

\bibitem{Durou2009}
Durou, J.D., Aujol, J.F., Courteille, F.: Integrating the Normal Field of a
  Surface in the Presence of Discontinuities, pp. 261--273.
\newblock Springer Berlin Heidelberg (2009)

\bibitem{durou2016}
Durou, J.D., Qu{\'e}au, Y., Aujol, J.F.: {Normal Integration -- Part I: A
  Survey} (2016)

\bibitem{sLLE}
Fang, Y., Vishwanathan, S., Sun, M., Ramani, K.: slle: Spherical locally linear
  embedding with applications to tomography.
\newblock In: CVPR (2011)

\bibitem{Frankot88amethod}
Frankot, R.T., Chellappa, R., Member, S.: A method for enforcing integrability
  in shape from shading algorithms.
\newblock IEEE Transactions on Pattern Analysis and Machine Intelligence
  \textbf{10}, 439--451 (1988)

\bibitem{Georghiades2003}
Georghiades, A.: Incorporating the torrance and sparrow model of reflectance in
  uncalibrated photometric stereo.
\newblock pp. 816--823 (2003)

\bibitem{Athos}
Georghiades, A.S.: Recovering 3-d shape and reflectance from a small number of
  photographs.
\newblock In: EGRW (2003)

\bibitem{Gotardo_2015_ICCV}
Gotardo, P.F.U., Simon, T., Sheikh, Y., Matthews, I.: Photogeometric scene flow
  for high-detail dynamic 3d reconstruction.
\newblock In: The IEEE International Conference on Computer Vision (ICCV)
  (2015)

\bibitem{Hayakawa1994}
Hayakawa, H.: Photometric stereo under a light-source with arbitrary motion.
\newblock JOSA-A  (1994)

\bibitem{SeitzExampleBasedPS}
Hertzmann, A., Seitz, S.: Example-based photometric stereo: shape
  reconstruction with general, varying brdfs.
\newblock PAMI  (2005)

\bibitem{ConsensusPS}
Higo, T., Matsushita, Y., Ikeuchi, K.: Consensus photometric stereo.
\newblock In: CVPR (2010)

\bibitem{hobson1931theory}
Hobson, E.W.: The theory of spherical and ellipsoidal harmonics.
\newblock CUP Archive (1931)

\bibitem{Ho2016}
Hoeltgen, L., Quéau, Y., Breuss, M., Radow, G.: {Optimised photometric stereo
  via non-convex variational minimisation (regular paper)}.
\newblock In: {British Machine Vision Conference (BMVC)} (2016)

\bibitem{IntrinsicDimensionEstimation}
Levina, E., Bickel, P.J.: Maximum likelihood estimation of intrinsic dimension.
\newblock In: NIPS (2004)

\bibitem{littwin2015spherical}
Littwin, E., Averbuch-Elor, H., Cohen-Or, D.: Spherical embedding of inlier
  silhouette dissimilarities  (2015)

\bibitem{Sato2}
Lu, F., Matsushita, Y., Sato, I., Okabe, T., Sato, Y.: Uncalibrated photometric
  stereo for unknown isotropic reflectances.
\newblock In: CVPR (2013)

\bibitem{mecca2015realistic}
Mecca, R., Rodol{\`a}, E., Cremers, D.: Realistic photometric stereo using
  partial differential irradiance equation ratios.
\newblock Computers \& Graphics  (2015)

\bibitem{Midorikawa_2016_CVPR}
Midorikawa, K., Yamasaki, T., Aizawa, K.: Uncalibrated photometric stereo by
  stepwise optimization using principal components of isotropic brdfs.
\newblock In: The IEEE Conference on Computer Vision and Pattern Recognition
  (CVPR) (2016)

\bibitem{Nillius2004}
Nillius, P., Eklundh, J.O.: Classifying materials from their reflectance
  properties.
\newblock In: European Conference on Computer Vision (ECCV), pp. 366--376
  (2004)

\bibitem{AttachedShadowCoding}
Okabe, T., Sato, I., Sato, Y.: {Attached shadow coding: Estimating surface
  normals from shadows under unknown reflectance and lighting conditions}.
\newblock In: ICCV (2009)

\bibitem{LDR}
Papadhimitri, T., Favaro, P.: A closed-form solution to uncalibrated
  photometric stereo via diffuse maxima.
\newblock IEEE Computer Society, Los Alamitos, CA, USA (2012)

\bibitem{PapadhimitriF13}
Papadhimitri, T., Favaro, P.: A new perspective on uncalibrated photometric
  stereo.
\newblock In: CVPR, pp. 1474--1481. IEEE Computer Society (2013)

\bibitem{Uncalibrated-Near-Light}
Papadhimitri, T., Favaro, P.: Uncalibrated near-light photometric stereo.
\newblock In: Proceedings of the British Machine Vision Conference. BMVA Press
  (2014)

\bibitem{Queau2015_multipoint}
Queau, Y., Lauze, F., Durou, J.D.: A L1-TV Algorithm for Robust Perspective
  Photometric Stereo with Spatially-Varying Lightings, SSVM (2015)

\bibitem{Queau2015}
Qu{\'e}au, Y., Lauze, F., Durou, J.D.: Solving uncalibrated photometric stereo
  using total variation.
\newblock Journal of Mathematical Imaging and Vision (JMIV)  (2015)

\bibitem{Queau_2016_CVPR}
Queau, Y., Mecca, R., Durou, J.D.: Unbiased photometric stereo for colored
  surfaces: A variational approach.
\newblock In: The IEEE Conference on Computer Vision and Pattern Recognition
  (CVPR) (2016)

\bibitem{Ramamoorthi&Hanrahan2001}
Ramamoorthi, R., Hanrahan, P.: On the relationship between radiance and
  irradiance: Determining the illumination from images of a convex lambertian
  object.
\newblock JOSA \textbf{18}(10) (2001)

\bibitem{SatoICCV2007}
Sato, I., Okabe, T., Yu, Q., Sato, Y.: {Shape reconstruction based on
  similarity in radiance changes under varying illumination}.
\newblock In: ICCV (2007)

\bibitem{LLE}
Saul, L.K., Roweis, S.T.: {Think globally, fit locally: unsupervised learning
  of low dimensional manifolds}.
\newblock J. Mach. Learn. Res. \textbf{4}, 119--155 (2003)

\bibitem{SelfCalibratingPS}
Shi, B., Matsushita, Y., Wei, Y., Xu, C., Tan, P.: P.: Self-calibrating
  photometric stereo.
\newblock In: CVPR (2010)

\bibitem{Simchony}
Simchony, T., Chellappa, R., Shao, M.: Direct analytical methods for solving
  poisson equations in computer vision problems.
\newblock IEEE Trans. Pattern Anal. Mach. Intell. \textbf{12}(5), 435--446
  (1990)

\bibitem{Smith2016}
Smith, W.A.P., Ramamoorthi, R., Tozza, S.: Linear Depth Estimation from an
  Uncalibrated, Monocular Polarisation Image (2016)

\bibitem{IsomapPaper}
Tenenbaum, J.B., de~Silva, V., Langford, J.C.: {A global geometric framework
  for nonlinear dimensionality reduction}.
\newblock Science  (2000)

\bibitem{MDS}
Torgerson, W.S.: Theory \& Methods of Scaling.
\newblock Wiley (1958)

\bibitem{Tozza2016}
Tozza, S., Mecca, R., Duocastella, M., Del Bue, A.: Direct differential
  photometric stereo shape recovery of diffuse and specular surfaces.
\newblock Journal of Mathematical Imaging and Vision  (2016)

\bibitem{Woodham1980}
{Woodham, R.J.}: {Photometric method for determining surface orientation from
  multiple images}.
\newblock OptEng pp. 139--144 (1980)

\bibitem{YiMaPS}
Wu, L., Ganesh, A., Shi, B., Matsushita, Y., Wang, Y., Ma, Y.: Robust
  photometric stereo via low-rank matrix completion and recovery.
\newblock In: ACCV (2011)

\bibitem{Xie_2015_CVPR}
Xie, W., Dai, C., Wang, C.C.L.: Photometric stereo with near point lighting: A
  solution by mesh deformation.
\newblock In: The IEEE Conference on Computer Vision and Pattern Recognition
  (CVPR) (2015)

\bibitem{YuilleIntegrability}
Yuille, A., Snow, D.: Shape and albedo from multiple images using
  integrability.
\newblock In: CVPR (1997)

\end{thebibliography}
}

\end{document}